\documentclass{article}

\usepackage{microtype}
\usepackage{graphicx}
\usepackage{booktabs} 

\usepackage{hyperref}

\usepackage{microtype}
\usepackage{graphicx}
\usepackage{subcaption}
\usepackage{booktabs} 

\usepackage{hyperref}

\usepackage{amsfonts}       
\usepackage{amsmath}
\usepackage{amsthm}
\usepackage{bbm}

\newcommand{\func}{\pi}

\newtheorem{lemma}{Lemma}
\newtheorem{definition}{Definition}
\newtheorem{theorem}{Theorem}

\newtheorem{corollary}{Corollary}

\newtheorem{remark}{Remark}
\newtheorem{assumption}{Assumption}

\allowdisplaybreaks

\makeatletter
\newcommand{\newreptheorem}[2]{\newtheorem*{rep@#1}{\rep@title}
	\newenvironment{rep#1}[1]{\def\rep@title{#2 \ref*{##1}}\begin{rep@#1}}{\end{rep@#1}}
}
\makeatother

\newreptheorem{lemma}{Lemma}
\newreptheorem{theorem}{Theorem}
\newreptheorem{claim}{Claim}
\newreptheorem{proposition}{Proposition}
\newreptheorem{corollary}{Corollary}

\usepackage[accepted]{icml2019}

\icmltitlerunning{Efficient learning of smooth probability functions from Bernoulli tests with guarantees}

\begin{document}

\twocolumn[
\icmltitle{Efficient learning of smooth probability functions from Bernoulli tests \\with guarantees}




\begin{icmlauthorlist}
\icmlauthor{Paul Rolland}{epfl}
\icmlauthor{Ali Kavis}{epfl}
\icmlauthor{Alex Immer}{epfl}
\icmlauthor{Adish Singla}{mpi-sws}
\icmlauthor{Volkan Cevher}{epfl}
\end{icmlauthorlist}

\icmlaffiliation{epfl}{Ecole Polytechnique F\'ed\'erale de Lausanne, Switzerland}
\icmlaffiliation{mpi-sws}{Max Planck Institute for Software Systems, Saarbr\"ucken, Germany}

\icmlcorrespondingauthor{Paul Rolland}{paul.rolland@epfl.ch}

\icmlkeywords{Machine Learning, ICML}

\vskip 0.3in
]



\printAffiliationsAndNotice{}  


\begin{abstract}
\vspace{2mm}
We study the fundamental problem of learning an unknown, smooth probability function via point-wise Bernoulli tests. We provide a scalable algorithm for efficiently solving this problem with rigorous guarantees. In particular, we prove the convergence rate of our posterior update rule to the true probability function in $L_2$-norm. Moreover, we allow the Bernoulli tests to depend on contextual features and provide a modified inference engine with provable guarantees for this novel setting. Numerical results show that the empirical convergence rates match the theory, and illustrate the superiority of our approach in handling contextual features over the state-of-the-art.

\end{abstract}


\section{Introduction}\label{sec.introduction}

One of the central challenges in machine learning relates to learning a continuous probability function from point-wise Bernoulli tests \cite{casella2002statistical,johnson2002multivariate}. Examples include, but are not limited to, clinical trials \cite{dersimonian1986meta}, recommendation systems \cite{mcnee2003interfaces}, sponsored search \cite{pandey2007handling}, and binary classification. Due to the curse of dimensionality, we often require a large number of tests in order to obtain an accurate approximation of the target function. It is thus  necessary to use a method that scalably constructs this approximation    with the number of tests.


A  widely used method for efficiently solving this problem is the Logistic Gaussian Process (LGP) algorithm \cite{tokdar2007posterior}. While this algorithm has no clear provable guarantees, it is shown to be very efficient in practice in approximating the target function. However, the time required for inferring the posterior distribution at some point grows cubicly with the number of tests, and can thus be inapplicable when the amount of data becomes large. There has been extensive work to resolve this cubic complexity associated with GP computations~\cite{rasmussen2004gaussian}. However, these methods require additional approximations on the posterior distribution, which impacts the efficiency and make the overall algorithm even more complicated, leading to further difficulties in establishing theoretical convergence guarantees.

Recently, \citet{goetschalckx2011continuous} tackled the issues encountered by LGP, and proposed a scalable inference engine based on Beta Processes called Continuous Correlated Beta Process (CCBP) for approximating the probability function. By scalable, we mean the algorithm complexity scales linearly with the number of tests. However, no theoretical analysis is provided, and the approximation error saturates as the number of tests becomes large (\textit{cf.}, section~\ref{sec.experiments-synthetic}). Hence, it is unclear whether provable convergence and scalability can be obtained simultaneously. 

This paper bridges this gap by designing a simple and scalable method for efficiently approximating the probability functions with provable convergence. Our algorithm constructs a posterior distribution that allows inference in linear time (w.r.t. the number of tests) and converges in $L_2$-norm to the true probability function (uniformly over the feature space), see Theorem~\ref{Static-convergence-theorem}.

In addition, we also allow the Bernoulli tests to depend on contextual parameters influencing the success probabilities. To ensure convergence of the approximation, these features need to be taken into account in the inference engine. We thus provide the first algorithm that efficiently treats these contextual features while performing inference, and retain guarantees. As a motivation for this setting, we demonstrate how this algorithm can efficiently be used for treating bias in the data \cite{agarwal2018reductions}.

\subsection{Basic model and the challenge}
In its basic form, we seek to learn an unknown, smooth function $\func:\mathcal{X} \rightarrow [0,1]$, $\mathcal{X} \subset \mathbb{R}^d$ from point-wise Bernoulli tests, where $d$ is the features space dimension. We model such tests as $s_i \sim \text{Bernoulli}(\func(x_i)),$ where $\sim$ means \textit{distributed as}, $x_i\in \mathcal{X}$, and we model our knowledge of $\func$ at point $x$ via a random variable $\tilde{\func}(x)$.  

Without additional assumptions, this problem is clearly hard, since experiments are performed only at points $\{x_i\}_{i=1,...,t}$, which constitute a negligible fraction of the space $\mathcal{X}$. In this paper, we will make the following assumption about the probability function:
\begin{assumption}
The function $\func$ is L-Lipschitz continuous, i.e., there exists a constant $L \in \mathbb{R}$ such that 
\begin{equation}
|\func(x) - \func(y)| \leq L \|x - y\|
\label{Lipschitz}
\end{equation}
$\forall x,y \in \mathcal{X}$ for some norm $\|.\|$ over $\mathcal{X}$.
\end{assumption}
In order to ensure convergence of $\tilde{\func}(x)$ to $\pi(x)$ for all $x\in \mathcal{X}$, we must design a way of sharing experience among variables using this smoothness assumption.

Our work uses a prior for $\func$ based on the Beta distribution and designs a simple sharing scheme to provably ensure convergence of the posterior.

\paragraph{Dynamic setting} In a more generic setting that we call ``dynamic setting," we assume that each Bernoulli test can be linearly influenced by some contextual features. Each experiment is then described by a quadruplet $S_i=(x_i, s_i, A_i, B_i)$ and we study the following simple model for its probability of success:
\begin{align}
Pr(s_{i} = 1) := A_i \func(x_i) + B_i. \label{eq.exp-dynamic}
\end{align}
We have to restrict $0 \leq B_i \leq 1$, $0 \leq A_i+B_i \leq 1$  to ensure that this quantity remains a probability given that $\func(x_i)$ lies in $[0,1]$. We assume that we have knowledge of estimates for $A_i$ and $B_i$ in expectation. 


Such contextual features naturally arise in real applications \cite{krause2011contextual}. For example, in the case of clinical trials \cite{dersimonian1986meta}, the goal is to learn the patient's probability of succeeding at an exercise with a given difficulty. A possible contextual feature can then be the state of fatigue of the patient, which can influence its success probability. Here, LGP algorithm could be used, but the contextual feature must be added as an additional parameter. We show that, if we know how this feature influences the Bernoulli tests, then we can achieve faster convergence.

\subsection{Our contributions} 
We summarize our contributions as follows: \\[-5mm]
\begin{enumerate}
\item We provide the first theoretical guarantees for the problem of learning a smooth probability function over a compact space using Beta Processes. \\[-5mm]
\item We provide an \textbf{efficient} and \textbf{scalable} algorithm that is able to handle contextual parameters explicitly influencing the probability function.\\[-5mm]
\item We demonstrate the efficiency of our method on synthetic data, and observe the benefit of treating contextual features in the inference. We also present a real-world application of our model.
\end{enumerate}

\paragraph{Roadmap} 
We  analyze the simple setting without contextual features (referred to as \textit{static} setting). We start by designing a Bayesian update rule for point-wise inference, and then include experience sharing in order to ensure $L_2$ convergence over the whole space with a provable rate. We then treat the dynamic setting in the same way, and finally demonstrate our theoretical findings via extensive simulations and on a case-study of clinical trials for rehabilitation (\textit{cf.}, Section~\ref{sec.experiments}).

\section{Related Work}\label{sec.related}

\paragraph{Correlated inference via GPs}
The idea of sharing experience of experiments between points with similar target function value is inspired by what was done with Gaussian Processes (GPs) \cite{williams1996gaussian}. GPs essentially define a prior over real-valued functions defined on a continuous space, and use a kernel function that represents how experiments performed at different points in the space are correlated.

GP-based models are not directly applicable to our problem setting given that our function $\func$ represents probabilities in the range $[0,1]$. For our problem setting, a popular approach is Logistic Gaussian Processes (LGP) \cite{tokdar2007posterior}---it learns an intermediate GP over the space $\mathcal{X}$ which is then squashed to the range $[0,1]$ via a logistic transformation. Experience sharing is then done by modeling the covariance between tests performed at different points through a predefined kernel. This allows constructing a covariance matrix between test points, which can be used to estimate the posterior distribution at any other sample point.


Gaussian Copula Processes (GCP) \cite{wilson2010copula} is another GP-based approach that learns a GP and uses a copula to map it to Beta distributions over the space.

More recently, Ranganath and Blei \cite{ranganath2017correlated} explored correlated random measures including correlated Beta-Bernoulli extension. However, GPs are still used in order to define these correlations.

There are at least two key limitations with these ``indirect" approaches: First, the posterior distributions after observing a Bernoulli outcome is analytically intractable, and needs to be approximated, e.g. using Laplace approximation \cite{tokdar2007posterior}. Second, the time complexity of prediction grows cubicly $\mathcal{O}(t^3)$ with respect to the number of samples $t$. There is extensive work to resolve this cubic complexity associated to GP computations~\cite{rasmussen2004gaussian}. However, these methods require additional approximations on the posterior distribution, which impacts the efficiency, and make the overall algorithm even more complicated, leading to further difficulties in establishing theoretical guarantees.

Methods based on GPs that take context variables into account have also been designed~\cite{krause2011contextual}. However, they simply allow for the use of specific kernels for these variables and still require an increase in the feature space dimension. In this work, by directly modifying the inference process, we compute a posterior that takes into account contextual features without increasing the feature space dimension.

\textbf{Correlated Beta Processes \ \ }
In contrast to GPs, it is very challenging to define correlated Beta distributions. The first work introducing a Beta process without using GPs is the one of Hjort (1990) \cite{hjort1990nonparametric}, but lacks the correlation aspect. Some other works studied multi-variate Beta distributions for simple settings considering only a few variables \cite{gupta1985three,olkin2003bivariate}.

Goetschalckx et al. (2011) \cite{goetschalckx2011continuous} proposed an approach named Continuous Correlated Beta Processes (CCBP) to deal with a continuous space of Beta distributions and to share experience between them via using a kernel. CCBP is shown to achieve results comparable to the state of the art approach based on LGP. Furthermore, it is shown that CCBP is much more time efficient---linear $\mathcal{O}(t)$ runtime  for CCBP in comparison to cubic $\mathcal{O}(t^3)$ runtime of GP-based methods.

CCBP approach has been used in several real-world application settings, e.g., for learning patient's fitness function in rehabilitation \cite{goetschalckx2011continuous}, learning the wandering behavior of people with dementia \cite{hoey2012modeling}, and in the application of analyzing genetic ancestry in admixed populations \cite{gompert2016continuous}.

However, the method presented in \cite{goetschalckx2011continuous}, by simply using a heuristic kernel, gives an approximation which does not converge to the target function as the number of samples increases. In order for the method to converge, this kernel must depend on the number of samples. In this paper, we provide an explicit kernel to use which ensures convergence of the approximated probability function to the target function with provable rate.



%
%










\section{Inference for the static setting}\label{sec:inference.static}

We start by analyzing the static setting, in which no contextual features influence the Bernoulli tests. We first design a Bayesian update rule for point-wise inference, then we propose an experience sharing method and prove convergence guarantees.

\subsection{Uncorrelated case: a Bayesian approach}

Suppose we do not use the smoothness assumption of $\func$. Then a naive solution is to model each random variable $\tilde{\func}(x)$ by the conjugate prior of the Bernoulli distribution, which is the Beta distribution. Then, starting from a prior $\tilde{\func}(x) \sim \text{Beta}(\alpha(x), \beta(x)) \ \forall x\in \mathcal{X}$, the Bayesian posterior $\tilde{\func}(x|\mathcal{S})$ conditioned on $\mathcal{S} = \{ (x_i, s_i) \}_{i=1,...,t}$ is:
\vspace{-2mm}
\begin{align*}
\tilde{\func}(x|\mathcal{S}) \sim  \text{Beta}&\left(\alpha(x) + \sum_{i=1}^t \delta_{s_i = 1}\delta_{x_i=x}, \right.\\
&\left.\beta(x) + \sum_{i=1}^t \delta_{s_i = 0}\delta_{x_i=x}\right)
\end{align*}
where $\delta_a$ is the Kronecker delta. 

This particular update scheme does not take smoothness assumption of the function into account and any experiment $S_i = (x_i, s_i)$ only influences the corresponding random variable $\tilde{\func}(x_i)$. In particular, if no experiment is performed at $x$, then our belief of $\func(x)$ remains unchanged. It is thus necessary to make use of the smoothness assumption.
\subsection{Leveraging smoothness of $\func$ via experience sharing}

\citet{goetschalckx2011continuous} propose a mechanism of experience sharing among correlated variables. To this purpose, they introduce a kernel $K:\mathcal{X} \times \mathcal{X} \rightarrow [0,1]$ where $K(x,x_i)$ indicates to what extent the experience for experiment at $x_i$ should be shared with any other point $x$. Indeed, thanks to the Lipschitz continuity assumption~\eqref{Lipschitz}, we expect close points to have similar probabilities. However, although the Beta distribution is the conjugate prior of the Bernoulli distribution, this conjugacy does not hold anymore when we use experience sharing. Instead of using the Bayesian posterior, we use the following update rule:
\begin{equation}
\begin{split}
 \tilde{\func}(x | \mathcal{S}) \sim  \text{Beta}&\left(\alpha(x) + \sum_{i=1}^t \delta_{s_i = 1}K(x, x_i), \right. \\
 &\left. \beta(x) + \sum_{i=1}^t \delta_{s_i = 0}K(x, x_i)\right).
 \end{split}
 \label{UpdateRule-static}
\end{equation}

With this update rule, the result of experiment $S_i$ influences all variables $\tilde{\func}(x)$ for which $K(x,x_i) > 0$, and the magnitude of influence is proportional to $K(x,x_i)$. Note that this update rule is no more Bayesian. However, all existing methods, including LGP and GCP, also involve non-Bayesian updates.


In \citet{goetschalckx2011continuous}, authors do not specify any particular choice of kernel function and the selection process is left as a heuristic. We show that proper selection of kernel is essential for convergence to the true underlying distributions at all points. In particular,  to ensure convergence in $L_2$ norm, this kernel must shrink as the number of observations increases (Algorithm~\ref{algo-static-setting}). We can see that our algorithm, called Smooth Beta Process (SBP), allows for fast inference at any point $x \in \mathcal{X}$, since it simply requires to find the tests that are performed at most $\Delta$ far from $x$, and compute the posterior distribution as in~\eqref{UpdateRule-static} depending on the number of successes and failures within these tests.



\begin{algorithm}[t!]
   \caption{Smooth Beta Process (SBP)}
   \label{algo-static-setting}
\begin{algorithmic}
   \STATE {\bfseries Input:} experiments points and observations $\mathcal{S} = \{x_i, s_i\}_{i=1,..,t}$, query point $x \in \mathcal{X}$, prior knowledge $\tilde{\func}(x) \sim \mathcal{B}(\alpha(x), \beta(x))$
   \STATE {\bfseries Output:} Posterior distribution $\tilde{\func}(x|\mathcal{S})$
   \STATE 1. Set $\Delta \propto \frac{1}{t^{\frac{1}{d+2}}}$
   \STATE 2. Compute the posterior as in~\eqref{UpdateRule-static} using kernel $K(x,x') = \delta_{\|x-x'\|\leq \Delta}$
\end{algorithmic}
\end{algorithm}

\begin{remark}
The particular dependence of the kernel on the number of samples ensuring optimal convergence is not trivial, and can only be found via a theoretical analysis of the model.
\end{remark}

\begin{theorem}\label{Static-convergence-theorem}
Let $\func:[0,1]^d \rightarrow [0,1]$ be $L$-Lipschitz continuous. Suppose we measure the results of experiments $\mathcal{S} = \{(x_i, s_i)\}_{i=1,...,t}$ where $s_i$ is a sample from a Bernoulli distribution with parameter $\func(x_i)$. Experiment points $\{x_i\}_{i=1,...,t}$ are assumed to be i.i.d. and uniformly distributed over the space. Then, starting with a uniform prior $\alpha(x) = \beta(x) = 1 \ \forall x\in [0,1]^d$, the posterior $\tilde{\func}(x|\mathcal{S})$ obtained from SBP uniformly converges in $L_2$-norm to $\func(x)$, i.e.
\begin{equation}
\sup_{x \in [0,1]^d}\mathbb{E}_{\mathcal{S}} \left(\mathbb{E}\left((\tilde{\func}(x|\mathcal{S}) - \func(x))^2\right)\right) = \mathcal{O}\left(L^{\frac{2d}{d+2}}t^{-\frac{2}{d+2}}\right),
\end{equation}
where the outer expectation is performed over experiment points $\{x_i\}_{i=1,...,t}$ and their results $\{s_i\}_{i=1,...,t}$.  SBP also computes point-wise posterior in time $\mathcal{O}(t)$.
\end{theorem}

\begin{remark}
This theorem provides an upper bound for the $L_2$ norm over any point of the space, and takes into account where the experiments are performed in the feature space. If all experiments are performed at the same point, then we recover the familiar square-root rate at that point \cite{ghosal1997review}, but we would not converge at points that are far away.
\end{remark}

\begin{remark}
The constraint on the input space being $[0,1]^d$ can easily be extended to any compact space $\mathcal{X} \subset \mathbb{R}^d$. This would simply modify the convergence rate by a factor equal to the volume of $\mathcal{X}$.
\end{remark}

 \begin{remark}
The dependence of the convergence rate on the feature space dimension is due to the curse of dimensionality, and the fact that we provide convergence uniformly over the whole space. However, this is not due to the particular algorithm we use, and we empirically show that LGP suffers the same dependence on the space dimension (see Section~\ref{sec.experiments}). Similar issues are also prevalent in GP optimization despite which great applications success has been obtained \cite{shahriari2016taking}. 
\end{remark}

In appendix~\ref{sec:classification}, we show how this algorithm naturally applies to binary classification. Restricted to classification, our algorithm becomes similar to the fixed-radius nearest neighbour algorithm ~\cite{chen2018explaining}, but the current framework allows for error quantification and precise prior injection.



\section{Inference for the dynamic setting}\label{sec:inference.dynamic}

We  analyze the dynamic setting where Bernoulli tests are influenced by contextual features as  in \eqref{eq.exp-dynamic}. 



\subsection{Uncorrelated case: a Bayesian approach}
As previously, we start by analyzing the uncorrelated case, i.e., how to update the distribution of $\tilde{\func}(x)$ conditioned on the outputs of experiments $\mathcal{S} = \{(s_i, x_i, A_i, B_i)\}_{i=1,...,t}$ all performed at $x$. 

Since experiments are not samples from Bernoulli variables with parameter $\func(x)$, Bayesian update is not straightforward but can be achieved using sums of Beta distributions, as shown in Theorem~\ref{bayesian-update-single-step}.

\begin{theorem}
\label{bayesian-update-single-step}
Suppose $\tilde{\func}(x) \sim \mathcal{B}(\alpha, \beta)$ and we observe the result of a sample $s\sim \text{Bernoulli}(A\func(x) + B)$. Then the Bayesian posterior for $\tilde{\func}(x)$  conditioned on this observation is given by
\begin{equation}
\tilde{\func}(x|s) \sim C_0 \mathcal{B}(\alpha+1, \beta) + C_1 \mathcal{B}(\alpha, \beta+1), 
\label{pointwiseBayesianUpdate}
\end{equation}
where  in the case of success ($s=1$), we have
\begin{equation*}
C_0 = \frac{B\beta}{B\beta + (A+B)\alpha}, \ \ C_1 = \frac{(A+B)\alpha}{B\beta + (A+B)\alpha}
\end{equation*} 
and in the case of failure ($s=0$), we have
\begin{align*}
C_0 &= \frac{(1-B)\beta}{(1-B)\beta+ (1-A-B)\alpha}, \\
C_1 &= \frac{(1-A-B)\alpha}{(1-B)\beta+ (1-A-B)\alpha}.
\end{align*}
In \eqref{pointwiseBayesianUpdate}, we mean that the density function of the posterior random variable $\tilde{\func}(x|s)$ is the weighted sum of the two density functions given on the right-hand side.
\end{theorem}

Then, by using this result recursively on a set of experiments $\mathcal{S} = \{(s_i,x_i,A_i,B_i)\}_{i=1,...,t}$, we can obtain a general update rule.
\begin{corollary}
\label{bayesian-update-several-steps}
Suppose $\tilde{\func}(x) \sim \mathcal{B}(\alpha, \beta)$ and we observe the outputs of experiments $\mathcal{S} = \{(s_i, x,A_i,B_i)\}_{i=1,...,t}$ where $s_i$'s are sampled from Bernoulli random variables, each with parameter $A_i\func(x) + B_i$. Then the Bayesian posterior $\tilde{\func}(x|\mathcal{S})$   is given by
\begin{equation}
\tilde{\func}(x|\mathcal{S}) \sim \sum_{i=0}^{t} C_i^t \mathcal{B}(\alpha+i, \beta + t-i)
\end{equation}
where $C_i^t$'s can be computed via an iterative procedure starting from $C_0^0 = 1$ and $\forall n=0,...,t$:
\begin{equation*}
C_i^{n+1} =  \frac{1}{E_s^n}(B_iC_i^n(\beta+n-i) + (A_i+B_i)C_{i-1}^n(\alpha + i - 1))
\end{equation*}
if $s_n = 1$; and
\begin{align*}
C_i^{n+1} =  \frac{1}{E_f^n}&((1-B_i)C_i^n(\beta+n-i) \\
&+ (1-A_i-B_i)C_{i-1}^n(\alpha + i - 1))
\end{align*}
if $s_n=0$. $E_s^n$ and $E_f^n$ are normalization factors that ensure $\sum_{i=0}^{n} C_i^{n} = 1 \forall n$. For simplicity of notation, we use $C_{-1}^n = C_{n+1}^n = 0$ $\forall n$.
\end{corollary}

This gives us a way of updating, in a fully Bayesian manner, the distribution of $\tilde{\func}(x)$ conditioned on observations of experiments performed at $x$. It involves a linear combination of Beta distributions, with coefficients depending on successes and contextual features.


\paragraph{Certainty invariance assumption} For simplicity, we will assume that if an event is certain (i.e., occurs with probability $1$), then context variables cannot lower this probability (i.e., $\forall x_i$ if $\pi(x_i)$, then $A_i \pi(x_i) + B_i = 1$). This is trivially equivalent to the constraint $A_i + B_i = 1$ $\forall i$. If, on the contrary, there is an impossibility invariance, i.e., context variables cannot make an impossible event possible, we can make a change of variable $f \leftrightarrow 1-f$ (i.e., invert the meanings of ``success'' and ``failure'') in order to satisfy the certainty invariance assumption.




In Corollary~\ref{bayesian-update-several-steps}, we see that the time complexity required for computing $C_i^t \ \forall i=0,...,t$ is $\mathcal{O}(t^2)$. However, in the particular case where $A_i = A$, $B_i = B$ $\forall i$ and $A+B = 1$, the update rule of Corollary~\ref{bayesian-update-several-steps} becomes much simpler, i.e., computable in time $\mathcal{O}(t)$:
\begin{corollary}
\label{bayesian-update-several-steps-simplified}
Suppose $\tilde{\func}(x) \sim \mathcal{B}(\alpha, \beta)$ and we observe the outputs of experiments $\mathcal{S} = \{(s_i, x,1-B,B)\}_{i=1,...,t}$ where $s_i \sim \text{Bernoulli}((1-B)\func(x) + B)$. Then the Bayesian posterior $\tilde{\func}(x|\mathcal{S})$  conditioned on these observations is given by
\begin{equation}
\tilde{\func}(x|\mathcal{S}) \sim \sum_{i=0}^{S} C_i^t \mathcal{B}(\alpha+i, \beta + t-i)
\end{equation}
where $S = \sum_{i=1}^t s_i$ is the total number of successes and
\begin{equation}
C_i^t \propto \binom{S}{i}(\alpha-1+i)!(\beta+t-1-i)!B^{S-i}
\end{equation}
$\forall i=0,...,S$. We can compute all $C_i^t$'s in time $\mathcal{O}(t)$ via the relation $C_{i+1}^t = \frac{(S-i)(\alpha+i)}{B(i+1)(\beta+t-1-i)}C_i^t$.
\end{corollary}

\begin{algorithm}[t!]
   \caption{Inference engine for the simplified dynamic setting: Constant $A,B$}
   \label{algo-dynamic-setting}
\begin{algorithmic}
   \STATE {\bfseries Input:} experiments descriptions $\mathcal{S} = \{(x_i, s_i,A,B)\}_{i=1,..,t}$, point of interest $x \in \mathcal{X}$, prior knowledge $\tilde{\func}(x) \sim \mathcal{B}(\alpha(x), \beta(x))$
   \STATE {\bfseries Output:} Posterior distribution $\tilde{\func}(x|\mathcal{S})$
   \STATE 1. Set $\Delta \propto \frac{1}{t^{\frac{1}{d+2}}}$
      \STATE 2. Build the set of neighboring experiments $\mathcal{S}_x = \{(x_i, s_i, A, B): \|x-x_i\|\leq \Delta\}$
   \STATE 3. Compute the posterior as in Corollary~\ref{bayesian-update-several-steps} (or ~\ref{bayesian-update-several-steps-simplified} if $A+B=1$) using the results of experiments $\mathcal{S}_x$ as if performed at $x$.
   \end{algorithmic}
\end{algorithm}

\subsection{Leveraging smoothness of $\func$  via experience sharing: Simplified setting}

We now introduce the use of correlations between samples in the update rule, in a similar way as done in the static setting. We first analyze a simplified setting where the contextual parameters are constant among all experiments, i.e. $A_i = A$, $B_i = B$ $\forall i$.

Due to the more complex form of the Bayesian update in the uncorrelated case, it turns out that the previous technique is not straightforward to apply. One way to introduce this idea of experience sharing would be to modify the Bayesian update rule of Theorem ~\ref{bayesian-update-single-step} as:
\begin{equation*}
\tilde{\func}(x|s_i) \sim C_0 \mathcal{B}(\alpha+K(x,x_i), \beta) + C_1 \mathcal{B}(\alpha, \beta+K(x,x_i))
\end{equation*}
where $K$ is the same kernel as defined previously. 

However, if $K(x,x')$ is real valued and that we apply such an update for each experiment, then it turns out that the number of terms required for describing the posterior distribution grows exponentially with the number of observations. In order to ensure the tractability of the posterior, we can restrict the kernel to values in $\{0,1\}$, e.g. $K(x,x') = \delta_{\|x-x'\|\leq \Delta}$ for some kernel width $\Delta$. This means that, each time we make an observation at $x_i$, all random variables $\tilde{\func}(x)$ with $\|x-x_i\| \leq \Delta$ are updated as if the same experiments had been performed at $x$ (Algorithm~\ref{algo-dynamic-setting}).


We provide guarantees for convergence of the posterior distribution generated by Algorithm~\ref{algo-dynamic-setting} under the certainty invariance assumption (Theorem~\ref{Dynamic-convergence-theorem}). This constraint on $A$ and $B$ is equivalent to saying that the contextual parameters necessarily increases the success probability. 

\begin{theorem}
Let $\func:[0,1]^d \rightarrow ]0,1]$ be $L$-Lipschitz continuous. Suppose we observe the results of experiments $\mathcal{S} = \{(x_i, s_i,1-B,B)\}_{i=1,...,t}$ where $s_i \sim \text{Bernoulli}((1-B)\func(x) + B)$. Experiment points $\{x_i\}_{i=1,...,t}$ are assumed to be i.i.d.\ uniformly distributed over the space. Then, starting with a uniform prior $\alpha(x) = \beta(x) = 1 \ \forall x\in [0,1]^d$, the posterior $\tilde{\func}(x|\mathcal{S})$ obtained from Algorithm~\ref{algo-dynamic-setting} uniformly converges in $L_2$-norm to $\func(x)$, i.e.,
\begin{align*}
\sup_{x \in [0,1]^d} \mathbb{E}_{\mathcal{S}} &\left(\mathbb{E}\left((\tilde{\func}(x | \mathcal{S}) - \func(x))^2\right)\right) \\
&\ \ \ \ \ \ = \mathcal{O}\left(L^{\frac{2d}{d+2}}((1-B)t)^{-\frac{2}{d+2}}\right).
\end{align*}
Moreover, Algorithm~\ref{algo-dynamic-setting} computes the point-wise posterior in time $\mathcal{O}(t)$.
\label{Dynamic-convergence-theorem}
\end{theorem}

\begin{remark}
We observe that adding the contextual parameter $B$ does not modify the convergence rate compared to the static case. By using other algorithms such as LGP, parameter $B$ should be added to the feature space, increasing its dimension by $1$, which impacts the convergence rate as we demonstrate in the sequel.
\end{remark}

\begin{algorithm}[t!]
   \caption{Contextual Smooth Beta Process (CSBP)}
   \label{algo-dynamic-setting-general}
\begin{algorithmic}
   \STATE {\bfseries Input:} experiments descriptions $\mathcal{S} = \{(x_i, s_i,A_i,B_i)\}_{i=1,..,t}$, point of interest $x \in \mathcal{X}$, prior knowledge $\tilde{\func}(x) \sim \mathcal{B}(\alpha(x), \beta(x))$
   \STATE {\bfseries Output:} Posterior distribution $\tilde{\func}(x|\mathcal{S})$
    \STATE 1. Set $\Delta \propto \frac{1}{t^{\frac{1}{d+2}}}$
      \STATE 2. Build the set of neighboring experiments \\ $\mathcal{S}_x = \{(x_i, s_i, A_i, B_i): \|x-x_i\|\leq \Delta\}$
   \STATE 3. Compute means $B = \frac{1}{|\mathcal{S}_x|}\sum_{i:  \|x-x_i\|\leq \Delta} B_i$, \\ $A = \frac{1}{|\mathcal{S}_x|}\sum_{i:  \|x-x_i\|\leq \Delta} A_i$
   \STATE 4. Compute the posterior as in Corollary~\ref{bayesian-update-several-steps} (or ~\ref{bayesian-update-several-steps-simplified} if $A+B=1$) using the results of experiments $\mathcal{S}_x$ as if performed at $x$, and with constant parameters $A, B$.
\end{algorithmic}
\end{algorithm}

\begin{figure*}[t!]
\centering
 \begin{subfigure}[t]{0.31\textwidth}
\includegraphics[width=5.5cm]{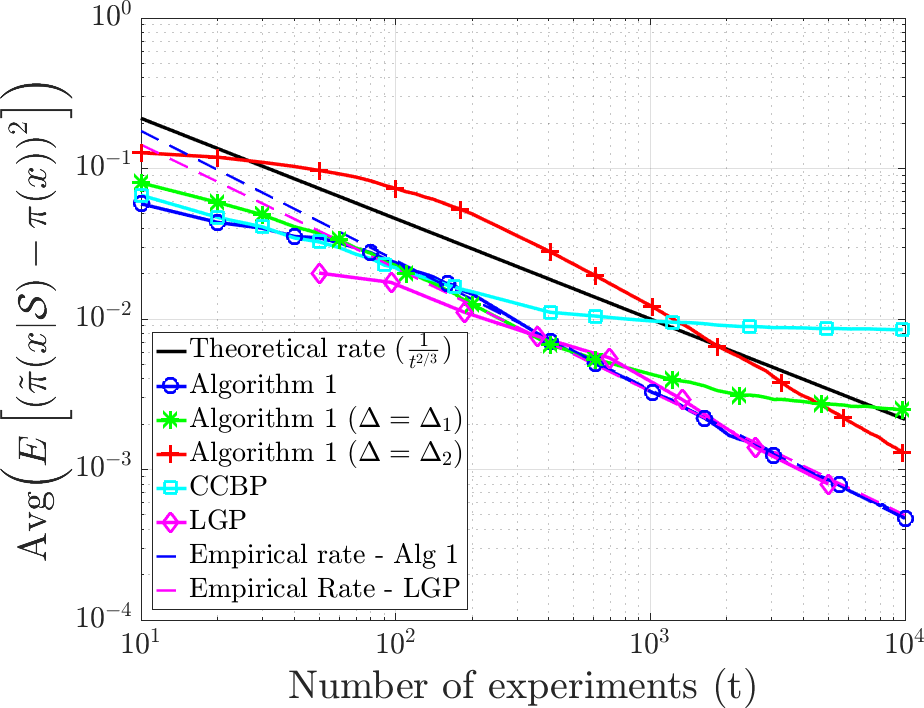}
\label{fig:1D_static_rate}
\end{subfigure}
\hspace{1mm}
\begin{subfigure}[t]{0.31\textwidth}
\includegraphics[width=5.2cm]{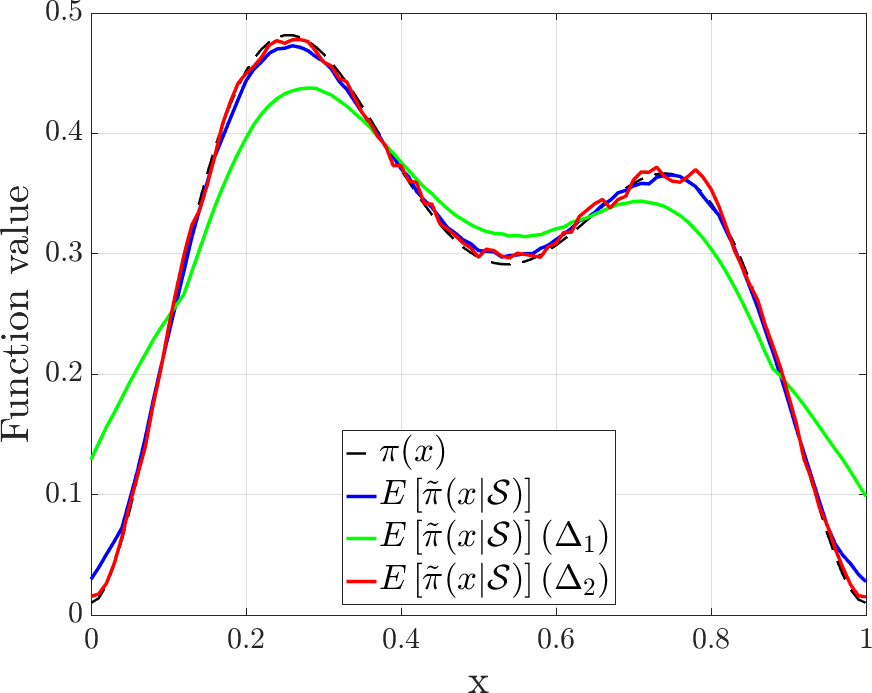}
\end{subfigure}
\begin{subfigure}[t]{0.31\textwidth}
\includegraphics[width=5.5cm]{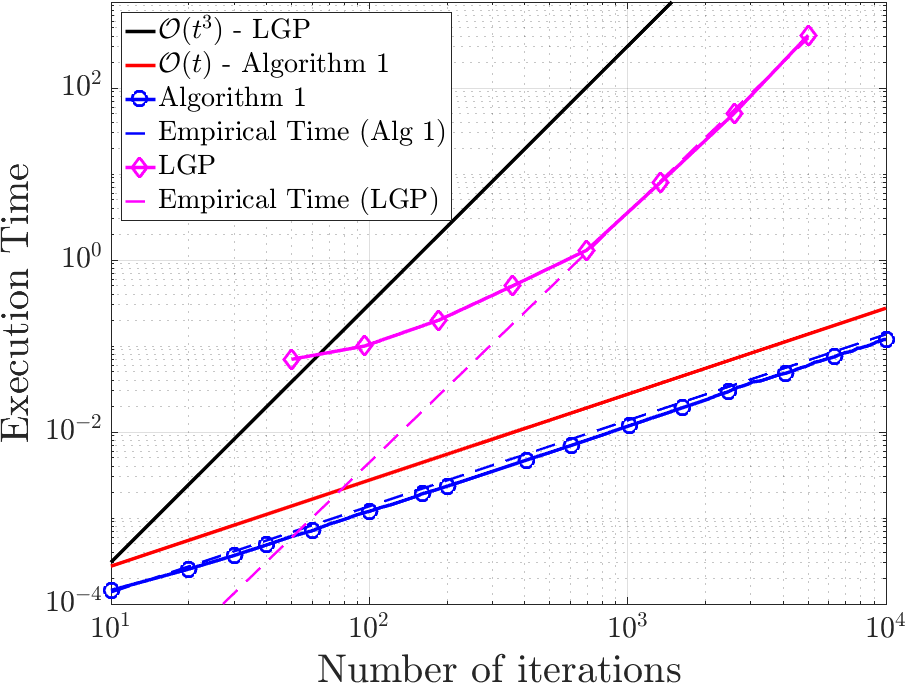}
\label{fig:1D_static_time}
\end{subfigure}
\caption{Left: $L_2$ error for 1D static setting. Middle: Mean posterior estimates $\mathbb{E}[\tilde{\func}(x|\mathcal{S})]$ generated by Algorithm~\ref{algo-static-setting} for different kernel widths. Right: Running time for Algorithm~\ref{algo-static-setting} and LGP}
\label{fig:1D_static}
\end{figure*}

\begin{figure*}[h!]
\centering
\begin{subfigure}[t]{0.45\textwidth}
\includegraphics[width=6cm]{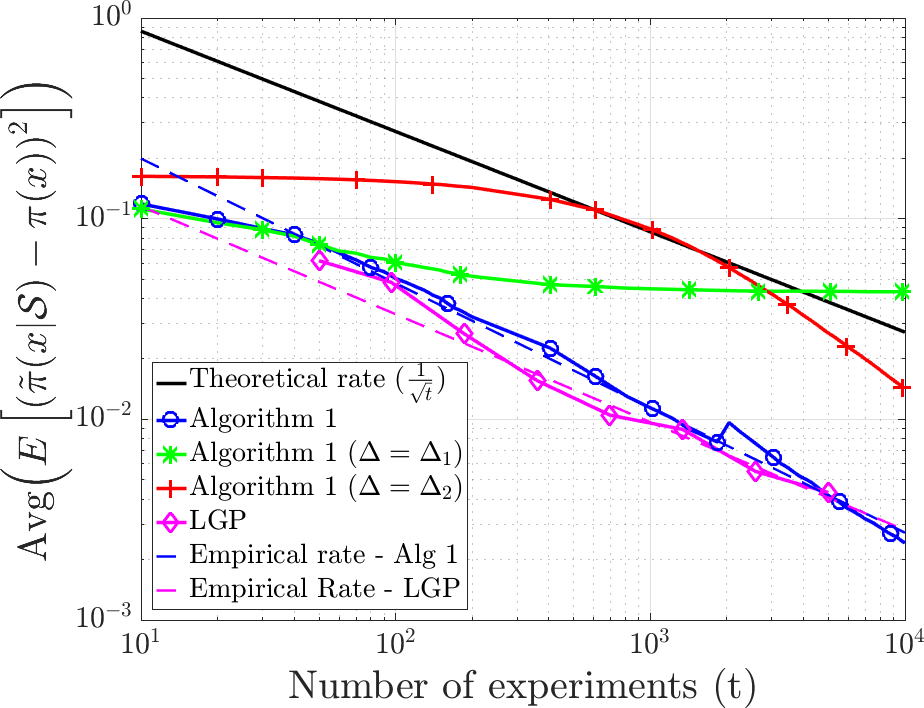}
\end{subfigure}
\begin{subfigure}[t]{0.45\textwidth}
\includegraphics[width=6cm]{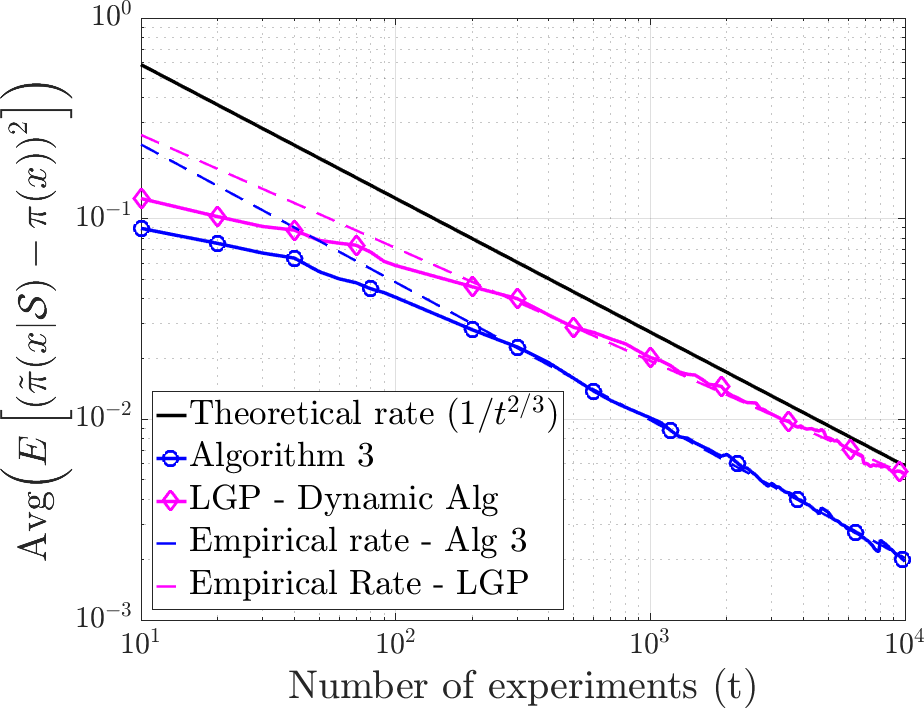}
\end{subfigure}
\caption{$L_2$ error of posterior $\mathbb{E}[\tilde{\func}(x|\mathcal{S})]$ for 2D static and 1D dynamic settings, averaged over all points $x\in \mathcal{X}$ versus number of samples.} \vspace{-4mm}
\label{fig:synthetic}
\end{figure*}

\subsection{Leveraging smoothness of $\func$ via experience sharing: General setting}

We finally analyze the general setting where contextual parameters are noisy and may vary among the experiments. Instead of using each $A_i$, $B_i$ in the update rule of the posterior, we can perform this update as if all experiments were performed with the same coefficients $A$ and $B$, which are the means of the coefficients $A_i$ and $B_i$ respectively. This approximation simplifies the analysis and does not influence the error  rate.

Algorithm \ref{algo-dynamic-setting-general}, called Contextual Smooth Beta Process (CSBP), is general and can be applied to any contextual parameters $A_i, B_i$ with no constraints. We provide guarantees for convergence of the posterior distribution generated by CSBP under the certainty invariance assumption (Theorem~\ref{Dynamic-convergence-theorem-general}).

\begin{theorem}
Let $\func:[0,1]^d \rightarrow ]0,1]$ be $L$-Lipschitz continuous. Suppose we observe the results of experiments $\mathcal{S} = \{(x_i, s_i,1-B_i,B_i)\}_{i=1,...,t}$ where $s_i \sim \text{Bernoulli}((1-(B_i+\epsilon_i))\func(x_i) + B_i+\epsilon_i)$, i.e., contextual features are noisy. We assume $\epsilon_i$'s are independent random variables with zero mean and variance $\sigma^2$. The points $\{x_i\}_{i=1,...,t}$ are assumed to be i.i.d.\ uniformly distributed over the space. Then, starting with a uniform prior $\alpha(x) = \beta(x) = 1 \ \forall x\in [0,1]^d$, the posterior $\tilde{\func}(x|\mathcal{S})$ obtained from Algorithm~\ref{algo-dynamic-setting-general} uniformly converges in $L_2$-norm to $\func(x)$, i.e.,
\begin{equation*}
\begin{split}
\sup_{x\in [0,1]^d}&\mathbb{E}_{\mathcal{S}} \left(\mathbb{E}\left((\tilde{\func}(x|\mathcal{S}) - \func(x))^2\right)\right) \\
&\ \ \ \ \ \ \ \ \ \ \ = \mathcal{O}\left(c(B,\sigma^2)L^{\frac{2d}{d+2}}t^{-\frac{2}{d+2}}\right),
\end{split}
\end{equation*}
where $c(B,\sigma^2)$ is a constant depending on $\{B_i\}_{i=1,...,t}$ and the noise $\sigma^2$. Moreover, CSBP computes the posterior in time $\mathcal{O}(t)$.
\label{Dynamic-convergence-theorem-general}
\end{theorem}

\section{Numerical experiments}\label{sec.experiments}
We devise a set of experiments to demonstrate the capabilities of our inference engine and validate the theoretical bounds for static and dynamic settings. We start with synthetic experiments in 1D and 2D, and finally reproduce the case study used in \cite{goetschalckx2011continuous} to show the efficiency of our dynamic algorithm.





\subsection{Synthetic examples}\label{sec.experiments-synthetic}

We construct a function $\func:\mathcal{X}\rightarrow [0,1]$, uniformly select points $\{x_i\}_{i=1,...,t}$, and sample $s_i \sim \text{Bernoulli}(\func(x_i)), \ i=1,...,t$. From these data, SBP constructs the posterior distributions $\tilde{\func}(x | \mathcal{S}) \ \forall x\in \mathcal{X}$. This experiment is performed both in 1D setting using a feature space $\mathcal{X} = [0,1]$, and in 2D with $\mathcal{X} = [0,1]^2$. We also apply LGP and CCBP (with fixed square exponential kernel) to this problem for comparison. Explicit forms of the chosen functions are presented in the Appendix.

For the dynamic setting, contextual parameters $\{ B_i \}_{i=1,...,t}$ are sampled independently and uniformly from $[0,1]$, and the tests are then performed by sampling $s_i \sim \text{Bernoulli}((1-B_i)\func(x_i) + B_i), \ i=1,...,t$. The posterior is constructed using CSBP. We also applied LGP to this dynamic setting by including the parameter $B$ as an additional feature. In order to evaluate $\func$, LGP returns the approximated distribution associated with $B=0$.

For the static setting (1D and 2D), Figures~\ref{fig:1D_static} (left) and ~\ref{fig:synthetic} (left) show the $L_2$ errors of the posterior distributions averaged across all $x\in \mathcal{X}$, and over $20$ runs, as functions of the number of samples $t$. We can observe the convergence upper bounds $\mathcal{O}(1 / t^{\frac{2}{3}})$ in 1D, and $\mathcal{O}(1 / \sqrt{t})$ in 2D as predicted by Theorem~\ref{Static-convergence-theorem}. 

We observe that LGP and our method perform similarly, as pointed out in~\cite{goetschalckx2011continuous}. However, running LGP takes significantly more time than our method since its time complexity is $\mathcal{O}(t^3)$ compared to $\mathcal{O}(t)$ for our algorithm, as demonstrated numerically in Figure~\ref{fig:1D_static} (right). We observe that CCBP saturates after some time since the kernel is independent of the number of samples.

Additionally, Figure~\ref{fig:1D_static} (left) demonstrates two sets of error curves for variations of Algorithm \ref{algo-static-setting}. To argue about optimality of kernel width specification, we run SBP with fixed kernel widths $\Delta_1 =  50^{-\frac{1}{d+2}}$ and $\Delta_2 = 500000^{-\frac{1}{d+2}}$. When $\Delta \ll t^{-\frac{1}{d+2}}$, the $L_2$ error initially decays at a slow rate and error remains larger than the optimal setting (green curves). On the contrary if we fix the kernel width $\Delta \gg t^{-\frac{1}{d+2}}$, the error saturates at early iterations (blue curves).

Figure~\ref{fig:1D_static} (middle) shows how the built posterior distribution approximates the true synthetic probability function by plotting the posterior mean over the space $\mathcal{X}$ for the different kernel widths. We observe that using a wide kernel ($\Delta_1$) leads to a posterior which is too smooth, due to experience oversharing. On the  other hand, using a narrow kernel ($\Delta_2$) leads to a highly non-smooth posterior, due to insufficient  sharing.

Figure~\ref{fig:synthetic} (right) similarly shows the $L_2$ error of the posterior distributions for the dynamic setting, also averaged over all $x\in \mathcal{X}$, and over $20$ runs. We again observe the convergence upper bounds $\mathcal{O}(1 / \sqrt{t})$ in 2D as predicted by Theorem~\ref{Dynamic-convergence-theorem-general}. We observe that our algorithm performs much better than LGP since it applies on a lower dimensional space.

\begin{figure*}[t!]
	\centering
		\begin{subfigure}[t]{0.45\textwidth}
			\includegraphics[width=6cm]{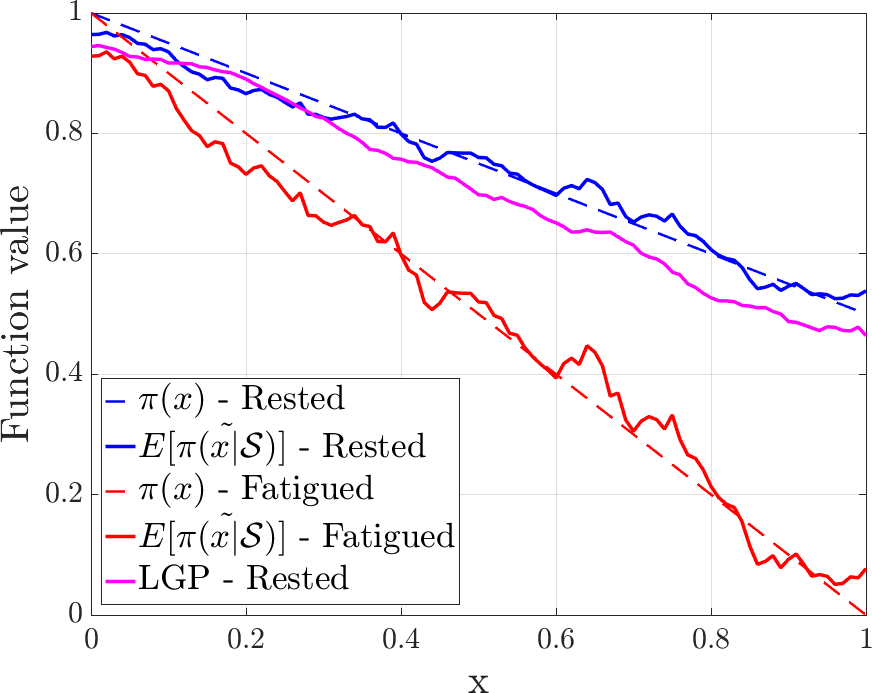}
			\label{fig:belief-multilevel-stroke}
		\end{subfigure}
		\begin{subfigure}[t]{0.45\textwidth}
			\includegraphics[width=6cm]{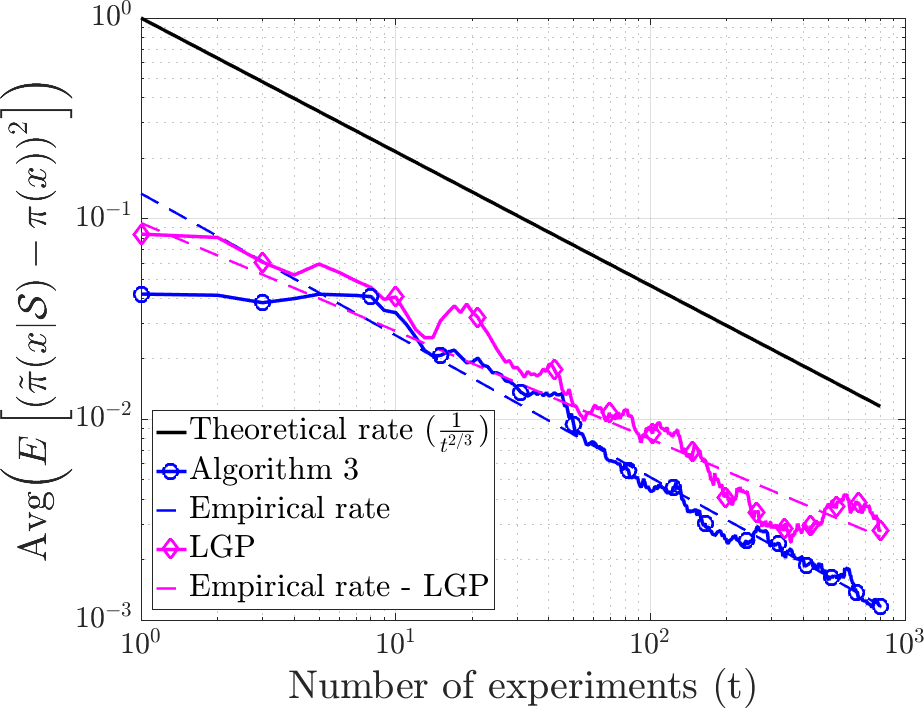}
			\label{fig:convergence-multilevel-stroke}
		\end{subfigure}
		\caption{Left: Mean posteriors $\mathbb{E}[\tilde{\func}(x|\mathcal{S})]$ for target functions representing rested (blue) and highest fatigue (red) states. Right: $L_2$ error for rested state averaged over all points versus sample size $t$.}
		\label{fig:rehab-main}
\end{figure*}

\subsection{Application to biased data: a case study from \citet{goetschalckx2011continuous}}\label{sec.experiments-realworld}
Handling biased data is currently one of the major problems in machine learning. In this section, we investigate how CSBP can treat bias by the mean of contextual features.


In line with \citet{goetschalckx2011continuous}, we conduct a case study with synthetic stroke rehabilitation data. The goal is to determine the probability that a patient succeeds in an exercise based on its difficulty. However, the patient can in some cases be fatigued, which influences the success probability and thus introduces a bias in the experiments.

Let $f(x)$ denote the success probability function for exercise with difficulty $x \in [0,1]$, when the patient is not fatigued. We assume that the patient has a certain level of fatigue $\alpha_f \in [0.5, 1]$, in which case his success probability function becomes $\alpha_f f$.

Note that the impossibility invariance assumption holds in this case since being fatigue cannot make possible a task which was already impossible. As mentioned previously, we can then simply make a change of variable in order to satisfy the certainty invariance assumption, and safely apply the dynamic algorithm.


Alternatively, by treating the level of fatigue as a new dimension in the feature space, LGP can be applied to the rehabilitation case study.

We assume that the difficulty of the exercise influences (in an unknown way) the success probability as $f(x) = 1-x$, $x \in [0,1]$. We construct a synthetic dataset by uniformly sampling exercise difficulties $x$ and fatigue levels $\alpha_f$, and then sampling the success from $f_{\alpha_f}(x)$.

Figure \ref{fig:rehab-main} shows the reconstructed success probability distributions when the patient is either not fatigued (rested state) or in the final fatigued state ($\alpha_f = 0.5$), as well as the $L_2$ error of the posterior for the rested state. Since LGP operates on a higher dimensional space, we observe that the $L_2$ error decays slower and the approximation of the target function for the rested state is worse than CSBP.

\section{Conclusions}\label{sec.conclusion}

In this paper, we build an inference engine for learning smooth probability functions from a set of Bernoulli experiments, which may be influenced by contextual features. We design an efficient and scalable algorithm for computing a posterior converging to the target function with provable rate, and demonstrate its efficiency on synthetic and real-world problems. These characteristics together with the simplicity of SBP make it a competitive tool compared to LGP, which has been shown to be an important tool in many real-world applications. We thus expect practitioners to apply this method to such problems.


\paragraph{Discussion and future work} The current analysis can only model a particular type of contextual influence, which modifies the success probability as $A_i\func(x_i) + B_i$. It turns out that Theorem~\ref{bayesian-update-single-step} can be generalized to any polynomial transformation of the success probability (i.e. $\sum_{j=0}^p a_i^{(j)}\func(x_i)^j$), allowing for a wider class of contextual influences.

Moreover, the theoretical framework we provide seems to be applicable to a large class of problems, such as risk tracking, Bandit setting, active learning, etc. Extending this model to such applications would also be an interesting research direction.




\section{Acknowledgement}

This work was supported by the Swiss National Science Foundation (SNSF) under  grant number 407540\_167319.

\bibliography{refs}
\bibliographystyle{icml2019}

\iftrue
\clearpage
 \onecolumn
 \appendix
 {\allowdisplaybreaks

\section{Proofs}\label{sec-app:proofs}

In this appendix, we provide all proofs for Theorems and Corollaries stated in the paper. We emphasize that we are aware of existing theoretical tools provided in ~\cite{van2008rates} and ~\cite{knapik2011bayesian}, but our approach is different and specific to the current setup.

\subsection{Proofs of point-wise Bayesian update in dynamic case}

\begin{theorem}
Suppose $\tilde{\func}(x) \sim \sum_{i=0}^n C_i^n \mathcal{B}(\alpha+i, \beta+n-i)$ with $\sum_{i=0}^n C_i^n = 1$, and we observe the result $s$ of a sample from a Bernoulli random variable with parameter $A\func(x) + B$. Then the Bayesian posterior for $\tilde{\func}(x)$  conditioned on this observation is:
\begin{equation}
\tilde{\func}(x|s) \sim \sum_{i=0}^{n+1} C_i^{n+1} \mathcal{B}(\theta, \alpha+i, \beta + n-i)
\end{equation}
where $\forall i=0,...,n+1$:
\begin{equation*}
C_i^{n+1} =  \frac{1}{E_s^n}(BC_i^n(\beta+n-i) + (A+B)C_{i-1}^n(\alpha + i - 1))
\end{equation*}
if $s=1$ and
\begin{equation*}
C_i^{n+1} =  \frac{1}{E_f^n}((1-B)C_i^n(\beta+n-i) + (1-A-B)C_{i-1}^n(\alpha + i - 1))
\end{equation*}
if $s=0$. $E_s^n$ and $E_f^n$ are normalization factors that ensure $\sum_{i=0}^n C_i^{n+1} = 1$. For simplicity of notation $C_{-1}^n = C_{n+1}^n = 0$ $\forall n$.
\end{theorem}

\begin{proof}

Suppose the observation is a success, i.e. $s=1$. Let $f_{\tilde{\func}(x)}:[0,1] \rightarrow [0,1]$ be the density function of the random variable $\tilde{\pi}(x)$, and let $f_{\tilde{\func}(x) | s=1}:[0,1] \rightarrow [0,1]$ be its the density function conditioned on this observation. Then,
\begin{align*}
f_{\tilde{\func}(x) | s=1}(\theta) &= \frac{Pr(s=1 | \tilde{\func}(x) = \theta) f_{\tilde{\func}(x)}(\theta)}{Pr(s=1)} \\
&\propto (A\theta + B) \sum_{i=0}^n C_i^n \mathcal{B}(\alpha+i, \beta+n-i) \\
&= (B(1-\theta) + (A + B)\theta) \sum_{i=0}^n C_i^n \frac{\theta^{\alpha+i-1}(1-\theta)^{\beta+n-i-1}}{\text{B}(\alpha+i, \beta+n-i)} \\
&= B \sum_{i=0}^n C_i^n \frac{\theta^{\alpha+i-1}(1-\theta)^{\beta+n-i}}{\text{B}(\alpha+i, \beta+n-i+1)}\frac{\text{B}(\alpha+i, \beta+n-i+1)}{\text{B}(\alpha+i, \beta+n-i)} \\
& + (A + B)\sum_{i=0}^n C_i^n \frac{\theta^{\alpha+i}(1-\theta)^{\beta+n-i-1}}{\text{B}(\alpha+i+1, \beta+n-i)}\frac{\text{B}(\alpha+i+1, \beta+n-i)}{\text{B}(\alpha+i, \beta+n-i)} \\
&= B \sum_{i=0}^n C_i^n \mathcal{B}(\alpha+i, \beta+n-i+1)\frac{\beta + n - i}{\alpha+\beta+n} \\
& + (A + B)\sum_{i=0}^n C_i^n \mathcal{B}(\alpha+i+1, \beta+n-i)\frac{\alpha + i}{\alpha+\beta+n} \\
&\propto \sum_{i=0}^{n+1} (BC_i^n(\beta+n-i) + (A+B)C_{i-1}^n(\alpha + i - 1)) \mathcal{B}(\alpha+i, \beta+n-i) \\
&\propto \sum_{i=0}^{n+1} C_i^{n+1} \mathcal{B}(\theta, \alpha+i, \beta + n-i)
\end{align*}

where $\text{B}$ is the Beta function, and satisfies $\frac{\text{B}(\alpha+1, \beta)}{\text{B}(\alpha, \beta)}= \frac{\alpha}{\alpha+\beta}$ and $\frac{\text{B}(\alpha, \beta+1)}{\text{B}(\alpha, \beta)}= \frac{\beta}{\alpha+\beta}$.

In order to ensure that this remains a probability distribution, coefficients $C_i^{n+1}$ must satisfy $\sum_{i=0}^{n+1} C_i^{n+1} = 1$. The result for $s=0$ can be showed similarly.
\end{proof}

Theorem~\ref{bayesian-update-single-step} is a special case of this result, for $n=0$. Corollary~\ref{bayesian-update-several-steps} directly follows from this theorem, by applying it recursively for each observations.

\begin{repcorollary}{bayesian-update-several-steps-simplified}
Suppose $\tilde{\func}(x) \sim \mathcal{B}(\alpha, \beta)$ and we observe the outputs of experiments $\mathcal{S} = \{(s_i, x,1-B,B)\}_{i=1,...,t}$ where $s_i \sim \text{Bernoulli}((1-B)\func(x) + B)$. Then the Bayesian posterior $\tilde{\func}(x|\mathcal{S})$  conditioned on these observations is given by
\begin{equation}
\tilde{\func}(x|\mathcal{S}) \sim \sum_{i=0}^{S} C_i^t \mathcal{B}(\alpha+i, \beta + t-i)
\end{equation}
where $S = \sum_{i=1}^t s_i$ is the total number of successes and
\begin{equation}
C_i^t \propto \binom{S}{i}(\alpha-1+i)!(\beta+t-1-i)!B^{S-i}
\label{CstBFastUpdate}
\end{equation}
$\forall i=0,...,S$. Using the relation $C_{i+1}^t = \frac{(S-i)(\alpha+i)}{B(i+1)(\beta+t-1-i)}C_i^t$, we can compute all $C_i^t$'s in time $\mathcal{O}(t)$.
\end{repcorollary}

\begin{proof}
We want to prove that the iterative process for computing the coefficients $C_i^t$'s in Corollary~\ref{bayesian-update-several-steps} ends with coefficients $C_i^t$'s of equation~\eqref{CstBFastUpdate}. We prove this by induction over $t$. For $t=0$, the result is obvious, since $S=0$, and $C_0^0=1$.

Now suppose the result is true for some time $n$ and let us prove that it remains true for time $n+1$. Let $S_n$ be the total number of successes observed up to time $n$, and let $s_{n+1}$ be the new observation at time $n+1$. Suppose $s_{n+1}=1$. Then $S_{n+1} = S_n +1$, and $\forall i=1,...,S_{n+1}$:
\begin{align*}
C_i^{n+1} &\propto BC_i^n(\beta+n-i) + C_{i-1}^n(\alpha+i-1) \\
&\propto \binom{S_n}{i}(\alpha-1+i)!(\beta+n-1-i)!B^{S_n+1-i}(\beta+n-i) \\
&+ \binom{S_n}{i-1}(\alpha-1+i-1)!(\beta+n-i)!B^{S_n+1-i}(\alpha+i-1) \\
&= \binom{S_{n+1}}{i}(\alpha-1+i)!(\beta+(n+1)-1-i)!B^{S_{n+1}-i} \\
\end{align*}

Similarly, if $s_{n+1}=0$, then $S_{n+1} = S_n$, and $\forall i=1,...,S_{n+1}$:
\begin{align*}
C_i^{n+1} &\propto (1-B)C_i^n(\beta+n-i) \\
&\propto \binom{S_{n+1}}{i}(\alpha-1+i)!(\beta+(n+1)-1-i)!B^{S_{n+1}-i}  \\
\end{align*}

In particular, we can see that the number of coefficients increases only when we observe a success.

\end{proof}

\subsection{Proof of convergence in the static case}\label{sec:proof-static}

\begin{reptheorem}{Static-convergence-theorem}\label{repthm:main-exp-theorem}
Let $\func:[0,1]^d \rightarrow [0,1]$ be $L$-Lipschitz continuous. Suppose we measure the results of experiments $\mathcal{S} = \{(x_i, s_i)\}_{i=1,...,t}$ where $s_i$ is a sample from a Bernoulli distribution with parameter $\func(x_i)$. Experiment points $\{x_i\}_{i=1,...,t}$ are assumed to be i.i.d. and uniformly distributed over the space. Then, starting with a uniform prior $\alpha(x) = \beta(x) = 1 \ \forall x\in [0,1]^d$, the posterior $\tilde{\func}(x|\mathcal{S})$ obtained from Algorithm~\ref{algo-static-setting} uniformly converges in $L_2$-norm to $\func(x)$, i.e.
\begin{equation}
\sup_{x \in [0,1]^d}\mathbb{E}_{\mathcal{S}} \left(\mathbb{E}\left((\tilde{\func}(x|\mathcal{S}) - \func(x))^2\right)\right) = \mathcal{O}\left(t^{-\frac{2}{d+2}}\right),
\end{equation}
where the outer expectation is performed over experiment points $\{x_i\}_{i=1,...,t}$ and their results $\{s_i\}_{i=1,...,t}$. Moreover, Algorithm~\ref{algo-static-setting} computes the posterior in time $\mathcal{O}(t)$.
\end{reptheorem}


\begin{proof}
For simplicity, suppose we start with a uniform prior for each $x$, i.e. $\tilde{\func}(x) \sim \mathcal{B}(1,1)$. Let $x\in \mathcal{X}$, $\Delta \in [0,1]$ be arbitrary. Suppose we fix the experiment points $X = \{x_i\}_{i=1,...,t}$ and that among these $t$ points, $n$ of them are at most $\Delta$ far from $x$ along all of $d$ dimensions. We assume without loss of generality that these points are $x_1,...,x_n$. Let $D_x$ be the random variable denoting the number of experiments occurring at most $\Delta$ far from $x$ along each dimension. Since we assume that experiment points $\{x_i\}_{i=1,...,t}$ are uniformly distributed over $[0,1]^d$, it follows that $D_x\sim \text{Bin}(t, \Delta^d)$.

Let $S_x$ denote the number of successes that occurred among these $n$ experiments. $S_x$ can be written as a $S_x = \sum_{i=1}^n s_i$ where $\bf{s}$ $= \{s_i\}_{i=1,...,n}$ are sampled independently, and $s_i\sim \text{Bernoulli}(\func(x_i))$ denotes whether experiment on $x_i$ was successful or not. Thus, $S_x$ follows a Poisson-Binomial distribution, and it follows:
\begin{equation}
\mathbb{E}(S_x | D_x=n) = \sum_{i=1}^n \func(x_i)
\end{equation}
and
\begin{equation}
\mathbb{E}(S_x^2 | D_x=n) = \sum_{i=1}^n \func(x_i)(1-\func(x_i)) + \left(\sum_{i=1}^n \func(x_i)\right)^2
\end{equation}

Note that after $s$ successes among $n$ experiments, the update rule~\ref{UpdateRule-static} leads to the posterior:
\begin{equation}
\tilde{\func}(x|\mathcal{S}) \sim \mathcal{B}(1+s, 1+n-s).
\end{equation}
Using the properties of the Beta distribution, we have:
\begin{equation}
\mathbb{E}(\tilde{\func}(x|\mathcal{S})|S_x=s, D_x=n) = \frac{s+1}{n+2}
\end{equation}
and
\begin{align*}
\mathbb{E}(\tilde{\func}(x|\mathcal{S})^2|S_x=s, D_x=n) &= \frac{(s+1)(n+1-s)}{(n+2)^2(n+3)} + \frac{(s+1)^2}{(n+2)^2} \\
&= \frac{(s+1)(s+2)}{(n+2)(n+3)} \\
&= \frac{s^2}{(n+2)^2} + \mathcal{O}\left(\frac{1}{n+1}\right)
\end{align*}

Therefore:
\begin{align*}
\mathbb{E}_{X,\bf{s}} &\left(\mathbb{E}\left((\tilde{\func}(x|\mathcal{S}) - \func(x))^2\right)\right) = \sum_{n=0}^t Pr(D_x=n) \mathbb{E}_{x_1,...,x_n}\left[\sum_{s=0}^n Pr(S_x=s | D_x=n)\left(\mathbb{E}(\tilde{\func}(x|\mathcal{S})^2|S_x=s,D_x=n) \right.\right.\\
&\left.\left. - 2\func(x)\mathbb{E}(\tilde{\func}(x|\mathcal{S})|S_x=s, D_x=n) + \func(x)^2\right)\right] \\
&= \sum_{n=0}^t Pr(D_x=n) \mathbb{E}_{x_1,...,x_n}\left[\sum_{s=0}^n Pr(S_x=s | D_x=n)\left(\frac{s^2}{(n+2)^2} + \mathcal{O}\left(\frac{1}{n+1}\right) - 2\func(x)\frac{s}{n+2} + \func(x)^2\right)\right] \\
&= \sum_{n=0}^t Pr(D_x=n) \mathbb{E}_{x_1,...,x_n}\left[\frac{1}{(n+2)^2}\left(\sum_{i=0}^n \func(x_i)(1-\func(x_i)) + \left(\sum_{i=0}^n \func(x_i)\right)^2 \right) \right. \\
&\left.\left. - \frac{2}{n+2}\func(x)\sum_{i=0}^n \func(x_i) + \func(x)^2  + \mathcal{O}\left(\frac{1}{n+1}\right) \right| \|x-x_i\|\leq \Delta \ \forall i = 1,...,n \right] \\
&= \sum_{n=0}^t Pr(D_x=n) \mathbb{E}_{x_1,...,x_n}\left[\frac{1}{(n+2)^2}\sum_{i=0}^n \func(x_i)(1-\func(x_i)) \right. \\
&\left. \left. + \frac{1}{(n+2)^2}\left(\sum_{i,j=0}^n (\func(x) - \func(x_i))(\func(x) - \func(x_j))\right)   + \mathcal{O}\left(\frac{1}{n+1}\right)\right| \|x-x_i\|\leq \Delta \ \forall i = 1,...,n \right] \\
&\leq \sum_{n=0}^t Pr(D_x=n) \left(\frac{1}{4(n+2)}+ \mathcal{O}\left(\frac{1}{n+1}\right)\right) + L^2\Delta^2 \\
&= L^2\Delta^2 + \mathcal{O}\left(\frac{1}{\Delta^d(t+1)}\right)
\end{align*}

Therefore, assuming $L > 0$, we can choose $\Delta = \frac{1}{L^{\frac{2}{d+2}}}t^{-\frac{1}{d+2}}$, and we obtain:
\begin{equation}
\mathbb{E}_{X,\bf{s}}\left(\mathbb{E}((\tilde{\func}(x) - \func(x))^2\right) = \mathcal{O}\left(L^{\frac{2d}{d+2}}t^{-\frac{2}{d+2}}\right)
\end{equation}
In particular, we observe that the smaller $L$, the larger $\Delta$. Indeed, the smoother the function, the more we can share experience between points $\{x_i\}$.
\end{proof}

\subsection{Proof of convergence in the simplified dynamic case}

\begin{reptheorem}{Dynamic-convergence-theorem}
Let $\func:[0,1]^d \rightarrow ]0,1]$ be $L$-Lipschitz continuous. Suppose we observe the results of experiments $\mathcal{S} = \{(x_i, s_i,1-B,B)\}_{i=1,...,t}$ where $s_i \sim \text{Bernoulli}((1-B_i)\func(x) + B_i)$. Experiment points $\{x_i\}_{i=1,...,t}$ are assumed to be uniformly distributed over the space. Then, $\forall x\in \mathcal{X}$, the posterior $\tilde{\func}(x|\mathcal{S})$ obtained from Algorithm~\ref{algo-dynamic-setting} converges in $L_2$-norm to $\func(x)$:
\begin{equation}
\mathbb{E}_{\mathcal{S}} \left(\mathbb{E}\left((\tilde{\func}(x) - \func(x))^2\right)\right) = \mathcal{O}\left(((1-B)t)^{-\frac{2}{d+2}}\right).
\end{equation}
Moreover, Algorithm~\ref{algo-dynamic-setting} computes the posterior in time $\mathcal{O}(t)$.\end{reptheorem}

\begin{proof}
Let $x\in \mathcal{X}$, $\Delta \in ]0,1]$ be arbitrary. Suppose we fix the experiment points $X$ and that among these $t$ points, $n$ of them are at most $\Delta$ far from $x$, i.e. $D_x = n$ where $D_x \sim \text{Bin}(t, \Delta^d)$ is the random variable as defined in~\ref{sec:proof-static}. We assume without loss of generality that these points are $x_1,...,x_n$. For simplicity, we treat the case where $\alpha = \beta = 1$, i.e. the prior for $\tilde{\func}(x)$ is uniform $\forall x\in \mathcal{X}$. Note that in this case, the coefficients $C_i$'s in Corollary~\ref{bayesian-update-several-steps-simplified} can be written as:
\begin{equation}
C_i^{n} =  \frac{1}{E'}\binom{n-i}{S-i}B^{S-i},
\end{equation}
$i=0,...,S$ where $E'$ is the normalization factor and $S$ is the number of observed successes.

\begin{align*}
\mathbb{E}_{\bf{s}}\left[\mathbb{E}(\tilde{\func}(x|\mathcal{S})) | D_x = n\right]  &= \sum_{s=0}^n Pr(S_x=s) \sum_{i=0}^s C_i^{n,s}(x) \frac{i+1}{n+2} \\
&= \sum_{s=0}^n Pr(S_x=s) \frac{\sum_{i=0}^s \binom{n-i}{s-i} B^{s-i} \frac{i+1}{n+2}}{\sum_{j=0}^s \binom{n-j}{s-j} B^{s-j}} \\
&= \sum_{s=0}^n Pr(S_x=s) \left(\frac{s+1}{n+2} - \frac{\sum_{i=0}^s \binom{n-s+i}{i} B^{i} \frac{i}{n+2}}{\sum_{j=0}^s \binom{n-s+j}{j} B^{j}}\right) \\
&= \sum_{s=0}^n Pr(S_x=s) \left(\frac{s+1}{n+2} - \frac{B}{1-B}\left(1-\frac{s+1}{n+2}\right)\left(1 - \frac{\binom{n+1}{s} B^{s}}{\sum_{j=0}^s \binom{n+1}{j} B^{j}(1-B)^{s-j}}\right)\right) \\
&= \frac{1+\sum_{i=1}^n(B + (1-B)\func(x_i))}{(n+2)(1-B)} - \frac{B}{1-B} \\
&+ \frac{B}{1-B}\sum_{s=0}^n Pr(S_x=s) \left(1-\frac{s+1}{n+2}\right)\frac{\binom{n+1}{s} B^{s}(1-B)^{n-s+1}}{\sum_{j=0}^s \binom{n+1}{j} B^{j}(1-B)^{n+1-j}} \\
&= \frac{\sum_{i=1}^n\func(x_i)}{n+2} + \frac{1 -2B}{(1-B)(n+2)} \\
&+ \frac{B}{1-B}\sum_{s=0}^n Pr(S_x=s) \left(1-\frac{s+1}{n+2}\right)\frac{\binom{n+1}{s} B^{s}(1-B)^{n-s+1}}{\sum_{j=0}^s \binom{n+1}{j} B^{j}(1-B)^{n+1-j}} \\
\end{align*}

At the fourth equality, we used the fact that $\sum_{j=0}^s \binom{n-j}{s-j} B^{s-j} = \sum_{j=0}^s \binom{n+1}{j} B^{j}(1-B)^{s-j}$, which can be shown by induction over $s$. We also used the following calculations:
\begin{align*}
\sum_{i=0}^s \binom{n-s+i}{i} B^{i}i &= (n-s+1)\sum_{i=1}^s \binom{n-s+i}{i-1} B^{i} \\
&= B(n-s+1)\sum_{i=0}^{s-1} \binom{n-s+1+i}{i} B^{i} \\
&= B(n-s+1)\left(\sum_{i=0}^{s-1} \binom{n-s+i}{i} B^{i} + \sum_{i=1}^{s-1} \binom{n-s+i}{i-1} B^{i} \right) \\
&= B(n-s+1)\left(\sum_{i=0}^{s} \binom{n-s+i}{i} B^{i} - \binom{n}{s}B^s + B\sum_{i=0}^{s-1} \binom{n-s+1+i}{i} B^{i} - \binom{n+1}{s}B^s \right)
\end{align*}
Therefore, by equaling lines $2$ and $4$ and using $\binom{n+1}{s+1} = \binom{n+1}{s}+\binom{n}{s}$, we get:
\begin{equation}
\sum_{i=0}^{s-1} \binom{n-s+1+i}{i} B^{i} = \frac{1}{1-B}\left(\sum_{i=0}^{s} \binom{n-s+i}{i} B^{i} - \binom{n+1}{s+1}B^s\right)
\label{binomialRed1}
\end{equation}
Thus:
\begin{equation}
\sum_{i=0}^s \binom{n-s+i}{i} B^{i}i = \frac{B}{1-B}\left(1-\frac{s+1}{n+2}\right)\left(\sum_{i=0}^{s} \binom{n-s+i}{i} B^{i} - \binom{n+1}{s+1}B^s\right)
\end{equation}

Let $Z \sim \text{Bin}(n+1, B)$. Then:
\begin{equation}
\left|\sum_{s=0}^n Pr(S_x=s) \left(1-\frac{s+1}{n+2}\right)\frac{\binom{n+1}{s} B^{s}(1-B)^{n-s+1}}{\sum_{j=0}^s \binom{n+1}{j} B^{j}(1-B)^{n+1-j}}\right| \leq \sum_{s=0}^t Pr(S_x=s) \frac{Pr(Z=s)}{Pr(Z\leq s)}
\end{equation}
We know that $\mathbb{E}(Z) = (n+1)B$ and $\mathbb{E}(S_x) = nB + \sum_{i=1}^n (1-B)\func(x_i)$. We then have:

\begin{align*}
\sum_{s=0}^n Pr(S_x=s) \frac{Pr(Z=s)}{Pr(Z\leq s)}&= \sum_{s=0}^{\frac{\mathbb{E}(Z) + \mathbb{E}(S_x)}{2}} Pr(S_x=s) \frac{Pr(Z=s)}{Pr(Z\leq s)} + \sum_{s=\frac{\mathbb{E}(Z) + \mathbb{E}(S_x)}{2}+1}^n Pr(S_x=s) \frac{Pr(Z=s)}{Pr(Z\leq s)} \\
&\leq Pr\left(S_x\leq \frac{\mathbb{E}(Z) + \mathbb{E}(S_x)}{2}\right) + 2Pr\left(Z\geq \frac{\mathbb{E}(Z) + \mathbb{E}(S_x)}{2}\right) \\
&\leq 3e^{-\frac{(\mathbb{E}(S_x) - \mathbb{E}(Z))^2}{2n}} \\
&\leq Ce^{-\frac{(1-B)^2\bar{\func}n}{2}}
\end{align*}
where $C\in \mathbb{R}$, $\bar{\func} = \frac{1}{n}\sum_{i=1}^n \func(x_i) > 0$. In the second step, we used $Pr(Z \leq s) \geq \frac{1}{2}$ for any $s\geq \mathbb{E}(Z)$. The last step follows from Hoeffding's inequality. So the previous upper bound decays exponentially to $0$. We thus have:
\begin{equation}
\mathbb{E}_{\bf{s}}\left[\mathbb{E}(\tilde{\func}(x|\mathcal{S}))| D_x=n\right] = \frac{\sum_{i=1}^n\func(x_i)}{n+2} + \frac{1-2B}{(1-B)(n+2)}
\end{equation}

We now bound the second moment of $\tilde{\func}(x|\mathcal{S})$. With the same notations as previously, we have:
\begin{align*}
\mathbb{E}_{\bf{s}}&\left[\mathbb{E}(\tilde{\func}(x|\mathcal{S})^2)  | D_x=n \right] = \sum_{s=0}^n Pr(S_x=s|D_x=n) \sum_{i=0}^s C_i^{n,s} \frac{(i+1)(i+2)}{(n+2)(n+3)} \\
&= \sum_{s=0}^n Pr(S_x=s|D_x=n)\left( \frac{(s+1)(s+2)}{(n+2)(n+3)} - 2\frac{s+1}{n+3}\sum_{i=0}^s C_{s-i}^{n,s} \frac{i}{n+2} + \sum_{i=0}^s C_{s-i}^{n,s} \frac{i(i-1)}{(n+2)(n+3)} + \mathcal{O}\left(\frac{1}{n+2}\right)\right) \\
&= \sum_{s=0}^n Pr(S_x=s|D_x=n)\left( \frac{(s+1)(s+2)}{(n+2)(n+3)} - 2\frac{B}{1-B} \frac{(s+1)(n-s+1)}{(n+2)(n+3)}\left(1 - \frac{\binom{n+1}{s} B^{s}}{\sum_{j=0}^s \binom{n+1}{j} B^{j}(1-B)^{s-j}} \right) \right. \\
& \left. +\frac{B^2}{1-B^2}\frac{(n-s+1)(n-s+2)}{(n+2)(n+3)}\left(\frac{1+B}{1-B} - \frac{2\binom{n+1}{s}\frac{B^{s+1}}{1-B} + \binom{n+2}{s}B^s + \binom{n+1}{s-1}B^{s-1}}{\sum_{j=0}^s \binom{n+1}{j} B^{j}(1-B)^{s-j}}\right) \right) \\
&= \frac{1}{(n+2)(n+3)}\sum_{s=0}^n Pr(S_x=s|D_x=n)\left( \frac{s^2}{(1-B)^2} - 2sn\frac{B}{(1-B)^2} + \frac{B^2}{(1-B)^2}n^2\right) + \mathcal{O}\left(\frac{1}{(1-B) (n+2)}\right) \\
&= \frac{1}{(1-B)^2(n+2)^2}\sum_{s=0}^n Pr(S_x=s|D_x=n)\left(\left(\sum_{i=1}^n(B + (1-B)\func(x_i))\right)^2 \right. \\
& \left. - 2Bn \sum_{i=1}^n\left(B + (1-B)\func(x_i)\right) + B^2n^2\right) + \mathcal{O}\left(\frac{1}{(1-B) (n+2)}\right) \\
&= \frac{1}{(n+2)^2}\sum_{s=0}^n Pr(S_x=s|D_x=n)\left(\sum_{i=1}^n\func(x_i)\right)^2 + \mathcal{O}\left(\frac{1}{(1-B) (n+2)}\right)
\end{align*}
where the four terms with denominator $\sum_{j=0}^s \binom{n+1}{j} B^{j}(1-B)^{s-j}$ in the third line can be shown to decay exponentially fast to $0$ similarly as previously. We computed $\sum_{i=0}^s C_{s-i}^{n,s} \frac{i(i-1)}{(n+2)(n+3)}$ in the second line using similar calculations as were done for $\sum_{i=0}^s C_{s-i}^{n,s} \frac{i}{n+2}$:
\begin{equation}
\sum_{i=0}^s \binom{n-s+i}{i} B^{i}i(i-1) = B^2(n-s+1)(n-s+2)\sum_{i=0}^{s-2} \binom{n-s+2+i}{i} B^{i}
\end{equation}
Using the identity $\binom{n+2}{k+2} = \binom{n}{k+2} + 2\binom{n}{k+1}+\binom{n}{k}$, we have:
\begin{align*}
\sum_{i=0}^{s-2} \binom{n-s+2+i}{i} B^{i} &= \sum_{i=0}^{s-2} \binom{n-s+i}{i} B^{i} + 2\sum_{i=1}^{s-2} \binom{n-s+i}{i-1} B^{i} + \sum_{i=2}^{s-2} \binom{n-s+i}{i-2} B^{i} \\
&= \sum_{i=0}^{s-2} \binom{n-s+i}{i} B^{i} + 2\sum_{i=1}^{s-3} \binom{n-s+i+1}{i} B^{i+1} + \sum_{i=2}^{s-3} \binom{n-s+i+2}{i} B^{i+2} \\
&= \sum_{i=0}^{s} \binom{n-s+i}{i} B^{i} - \binom{n-1}{s-1}B^{s-1} - \binom{n}{s}B^{s} \\
&+ 2B\sum_{i=1}^{s-1} \binom{n-s+i+1}{i} B^{i} - 2\binom{n-1}{s-2}B^{s-1} - 2\binom{n}{s-1}B^{s} \\
&+ B^2\sum_{i=2}^{s-2} \binom{n-s+i+2}{i} B^{i} - \binom{n-1}{s-3}B^{s-1} - \binom{n}{s-2}B^{s}
\end{align*}

Therefore, by isolating the term $\sum_{i=0}^{s-2} \binom{n-s+2+i}{i} B^{i}$, simplifying binomial coefficients and using equation~\eqref{binomialRed1}, we get:
\begin{align*}
\sum_{i=0}^{s-2} \binom{n-s+2+i}{i} B^{i} &= \frac{1}{1-B^2}\left(\frac{1+B}{1-B}\sum_{i=0}^{s} \binom{n-s+i}{i} B^{i} - \binom{n+2}{s}B^{s} - 2\binom{n+1}{s}\frac{B^{s+1}}{1-B} \right. \\
&\left.- \binom{n+1}{s-1}B^{s-1}\right)
\end{align*}

Thus:
\begin{align*}
\mathbb{E}_{\bf{s}}&\left[\mathbb{E}((\tilde{\func}(x|\mathcal{S})-\func(x))^2)  | D_x=n \right] = \mathbb{E}_{\bf{s}}\left[\mathbb{E}(\tilde{\func}(x|\mathcal{S})^2)  | D_x=n \right] - 2\func(x)\mathbb{E}_{\bf{s}}\left[\mathbb{E}\tilde{\func}(x|\mathcal{S})  | D_x=n \right] + \func(x)^2 \\
&= \sum_{s=0}^n Pr(S_x=s|D_x=n)\left(\left(\frac{\sum_{i=1}^n\func(x_i)}{n+1}\right)^2 -2\func(x)\frac{\sum_{i=1}^n\func(x_i)}{n+2} + \func(x)^2\right)+ \mathcal{O}\left(\frac{1}{(1-B) (n+2)}\right) \\
&= \sum_{s=0}^n Pr(S_x=s|D_x=n)\left(\frac{\sum_{i,j=1}^n(\func(x_i)-\func(x))(\func(x_j)-\func(x))}{(n+1)^2}\right) + \mathcal{O}\left(\frac{1}{(1-B) (n+2)}\right) \\
&\leq L^2\Delta^2 + \mathcal{O}\left(\frac{1}{(1-B) (n+2)}\right)
\end{align*}

By taking the expectation over $X$, we finally get:
\begin{align*}
\mathbb{E}_{X,\bf{s}}&\left[\mathbb{E}((\tilde{\func}(x|\mathcal{S})-\func(x))^2)\right] = \sum_{n=0}^t Pr(D_x=n)\mathbb{E}_{\bf{s}}\left[\mathbb{E}((\tilde{\func}(x|\mathcal{S})-\func(x))^2)  | D_x=n\right] \\
&\leq L^2\Delta^2  + \mathcal{O}\left(\frac{1}{(1-B)\Delta^d t}\right)
\end{align*}
If we choose $\Delta = \frac{1}{L^{\frac{2}{d+2}}}((1-B)t)^{-\frac{1}{d+2}}$, we obtain the desired result.

\end{proof}

\subsection{Proof of convergence in the general dynamic case}

\begin{reptheorem}{Dynamic-convergence-theorem-general}
Let $\func:[0,1]^d \rightarrow ]0,1]$ be $L$-Lipschitz continuous. Suppose we observe the results of experiments $\mathcal{S} = \{(x_i, s_i,1-B_i,B_i)\}_{i=1,...,t}$ where $s_i \sim \text{Bernoulli}((1-(B_i+\epsilon_i))\func(x_i) + B_i+\epsilon_i)$, i.e. contextual features are noisy. We assume $\epsilon_i$'s are independent random variables with zero mean and variance $\sigma^2$. Experiment points $\{x_i\}_{i=1,...,t}$ are assumed to be uniformly distributed over the space. Then, $\forall x\in \mathcal{X}$, the posterior $\tilde{\func}(x|\mathcal{S})$ obtained from Algorithm~\ref{algo-dynamic-setting-general} converges in $L_2$-norm to $\func(x)$ :
\begin{equation}
\mathbb{E}_{\mathcal{S}} \left(\mathbb{E}\left((\tilde{\func}(x|\mathcal{S}) - \func(x))^2\right)\right) = \mathcal{O}\left(c(B,\sigma^2)t^{-\frac{2}{d+2}}\right),
\end{equation}
where $c(B,\sigma^2)$ is a constant depending on $\{B_i\}_{i=1,...,t}$ and the noise $\sigma^2$. Moreover, Algorithm~\ref{algo-dynamic-setting-general} computes the posterior in time $\mathcal{O}(t)$.
\end{reptheorem}

\begin{proof}
The proof of theorem~\ref{Dynamic-convergence-theorem} can be completely adapted to this new setting. Let $x\in \mathcal{X}$, $\Delta \in [0,1]$ be arbitrary. Suppose we fix the experiment points $X$ and that among these $t$ points, $n$ of them are at most $\Delta$ far from $x$. We assume without loss of generality that these points are $x_1,...,x_n$. We then define $B_X = \frac{1}{n}\sum_{i=1}^n B_i$.

\begin{align*}
\mathbb{E}_{\bf{s}}\left[\mathbb{E}(\tilde{\func}(x|\mathcal{S})) | D_x = n\right]  &= \sum_{s=0}^n Pr(S_x=s) \sum_{i=0}^s C_i^{n,s}(x) \frac{i+1}{n+2} \\
&= \sum_{s=0}^n Pr(S_x=s) \frac{\sum_{i=0}^s \binom{n-i}{s-i} B_X^{s-i} \frac{i+1}{n+2}}{\sum_{j=0}^s \binom{n-j}{s-j} B_X^{s-j}} \\
&= \sum_{s=0}^n Pr(S_x=s) \left(\frac{s+1}{n+2} - \frac{\sum_{i=0}^s \binom{n-s+i}{i} B_X^{i} \frac{i}{n+2}}{\sum_{j=0}^s \binom{n-s+j}{j} B_X^{j}}\right) \\
&= \sum_{s=0}^n Pr(S_x=s) \left(\frac{s+1}{n+2} - \frac{B_X}{1-B_X}\left(1-\frac{s+1}{n+2}\right)\left(1 - \frac{\binom{n+1}{s} B_X^{s}}{\sum_{j=0}^s \binom{n+1}{j} B_X^{j}(1-B_X)^{s-j}}\right)\right) \\
&= \frac{1+\sum_{i=1}^n(B_i + \epsilon_i + (1-B_i - \epsilon_i)\func(x_i))}{(n+2)(1-B_X)} - \frac{B_X}{1-B_X} \\
&+ \frac{B_X}{1-B_X}\sum_{s=0}^n Pr(S_x=s) \left(1-\frac{s+1}{n+2}\right)\frac{\binom{n+1}{s} B_X^{s}(1-B_X)^{n-s+1}}{\sum_{j=0}^s \binom{n+1}{j} B_X^{j}(1-B_X)^{n+1-j}} \\
&= \frac{\sum_{i=1}^n(1-B_i)\func(x_i)}{(1-B_X)(n+2)} + \frac{1 -2B+ \sum_{i=1}^n\epsilon_i(1-\func(x_i))}{(1-B_X)(t+2)} \\
&+ \frac{B_X}{1-B_X}\sum_{s=0}^n Pr(S_x=s) \left(1-\frac{s+1}{n+2}\right)\frac{\binom{n+1}{s} B_X^{s}(1-B_X)^{n-s+1}}{\sum_{j=0}^s \binom{n+1}{j} B_X^{j}(1-B_X)^{n+1-j}} \\
\end{align*}

Let $Z \sim \text{Bin}(n+1, B_X)$. Then:
\begin{equation}
\left|\sum_{s=0}^n Pr(S_x=s) \left(1-\frac{s+1}{n+2}\right)\frac{\binom{n+1}{s} B_X^{s}(1-B_X)^{n-s+1}}{\sum_{j=0}^s \binom{n+1}{j} B_X^{j}(1-B_X)^{n+1-j}}\right| \leq \sum_{s=0}^t Pr(S_x=s) \frac{Pr(Z=s)}{Pr(Z\leq s)}
\end{equation}
We know that $\mathbb{E}(Z) = (n+1)B_X$ and $\mathbb{E}(S_x) = nB_X + \sum_{i=1}^n (1-B_i)\func(x_i) + \sum_{i=1}^n \epsilon_i(1 - \func(x_i))$. Since $\mathbb{E}(\epsilon_i) =0$, then $\mathbb{E}(S_x) - \mathbb{E}(Z)$ will also increase linearly with $n$ and thus the previous upper bound also decreases exponentially with $n$ to $0$ with very high probability. We thus have:

\begin{align*}
\mathbb{E}_{\bf{s},\epsilon}\left[\mathbb{E}(\tilde{\func}(x|\mathcal{S}))| D_x=n\right] &= \frac{\sum_{i=1}^n(1-B_i)\func(x_i)}{(1-B_X)(n+2)} + \mathbb{E}_{\epsilon}\left[\frac{\sum_{i=1}^n\epsilon_i(1-\func(x_i))}{(1-B_X)(t+2)}\right] + \mathcal{O}\left(\frac{1}{(1-B_X)(n+2)}\right) \\
&= \frac{\sum_{i=1}^n(1-B_i)\func(x_i)}{(1-B_X)(n+2)} + \mathcal{O}\left(\frac{1}{(1-B_X)(n+2)}\right)
\end{align*}

We now bound the second moment of $\tilde{\func}(x|\mathcal{S})$. With the same notations as previously, we have:
\begin{align*}
\mathbb{E}_{S}&\left[\mathbb{E}(\tilde{\func}(x|\mathcal{S})^2)  | D_x=n \right] = \sum_{s=0}^n Pr(S_x=s|D_x=n) \sum_{i=0}^s C_i^{n,s} \frac{(i+1)(i+2)}{(n+2)(n+3)} \\
&= \sum_{s=0}^n Pr(S_x=s|D_x=n)\left( \frac{(s+1)(s+2)}{(n+2)(n+3)} - 2\frac{s+1}{n+3}\sum_{i=0}^s C_{s-i}^{n,s} \frac{i}{n+2} + \sum_{i=0}^s C_{s-i}^{n,s} \frac{i(i-1)}{(n+2)(n+3)} + \mathcal{O}\left(\frac{1}{n+2}\right)\right) \\
&= \sum_{s=0}^n Pr(S_x=s|D_x=n)\left( \frac{(s+1)(s+2)}{(n+2)(n+3)} - 2\frac{B_X}{1-B_X} \frac{(s+1)(n-s+1)}{(n+2)(n+3)}\left(1 - \frac{\binom{n+1}{s} B_X^{s}}{\sum_{j=0}^s \binom{n+1}{j} B_X^{j}(1-B_X)^{s-j}} \right) \right. \\
& \left. +\frac{B_X^2}{1-B_X^2}\frac{(n-s+1)(n-s+2)}{(n+2)(n+3)}\left(\frac{1+B_X}{1-B_X} - \frac{2\binom{n+1}{s}\frac{B_X^{s+1}}{1-B_X} + \binom{n+2}{s}B_X^s + \binom{n+1}{s-1}B_X^{s-1}}{\sum_{j=0}^s \binom{n+1}{j} B_X^{j}(1-B_X)^{s-j}}\right) \right) \\
&= \frac{1}{(n+2)(n+3)}\sum_{s=0}^n Pr(S_x=s|D_x=n)\left( \frac{s^2}{(1-B_X)^2} - 2sn\frac{B_X}{(1-B_X)^2} + \frac{B_X^2}{(1-B_X)^2}n^2\right) + \mathcal{O}\left(\frac{1}{(1-B_X) (n+2)}\right) \\
&= \frac{1}{(1-B_X)^2(n+2)^2}\sum_{s=0}^n Pr(S_x=s|D_x=n)\left(\sum_{i=1}^n\left(B_i+\epsilon_i + (1-B_i-\epsilon_i)\func(x_i)\right)^2 \right. \\
& \left. - 2B_Xn \sum_{i=1}^n\left(B_i+\epsilon_i + (1-B_i-\epsilon_i)\func(x_i)\right) + B_X^2n^2\right) + \mathcal{O}\left(\frac{1}{(1-B_X) (n+2)}\right) \\
&= \frac{1}{(1-B_X)^2(n+2)^2}\sum_{s=0}^n Pr(S_x=s|D_x=n)\left(2\sum_{i,j=1}^n\epsilon_i(1-\func(x_i))\func(x_j) + \sum_{i,j=1}^n\epsilon_i\epsilon_j(1-\func(x_i))(1-\func(x_j)) \right. \\
&\left. + \left(\sum_{i=1}^n(1-B_i)\func(x_i)\right)^2\right) + \mathcal{O}\left(\frac{1}{(1-B_X) (n+2)}\right) \\
\end{align*}
where the four terms with denominator $\sum_{j=0}^s \binom{n+1}{j} B_X^{j}(1-B_X)^{s-j}$ in the third line can be shown to decay exponentially fast to $0$ similarly as previously. Taking the expectation over $\epsilon$, we then get:
\begin{align*}
\mathbb{E}_{S,\epsilon}&\left[\mathbb{E}(\tilde{\func}^t(x)^2)  | D_x=n \right] = \sum_{s=0}^n Pr(S_x=s|D_x=n)\left(\frac{\sum_{i=1}^n(1-B_i)\func(x_i)}{(1-B_X)(n+1)}\right)^2 \\
&+ \mathcal{O}\left(\frac{1}{(1-B_X) (n+1)} + \frac{\sigma^2}{(1-B_X)^2 (n+1)}\right)
\end{align*}

Thus:
\begin{align*}
\mathbb{E}_{\bf{s},\epsilon}&\left[\mathbb{E}((\tilde{\func}(x|\mathcal{S})-\func(x))^2)  | D_x=n \right] = \mathbb{E}_{\bf{s},\epsilon}\left[\mathbb{E}(\tilde{\func}(x|\mathcal{S})^2)  | D_x=n \right] - 2\func(x)\mathbb{E}_{S,\epsilon}\left[\mathbb{E}(\tilde{\func}(x|\mathcal{S}))  | D_x=n \right] + \func(x)^2 \\
&= \sum_{s=0}^n Pr(S_x=s|D_x=n)\left(\left(\frac{\sum_{i=1}^n(1-B_i)\func(x_i)}{(1-B_X)(n+1)}\right)^2 -2\func(x)\frac{\sum_{i=1}^n(1-B_i)\func(x_i)}{(1-B_X)(n+2)} + \func(x)^2\right)\\
&+ \mathcal{O}\left(\frac{1}{(1-B_X) (n+2)} + \frac{\sigma^2}{(1-B_X)^2 (n+2)}\right) \\
&= \sum_{s=0}^n Pr(S_x=s|D_x=n)\left(\frac{\sum_{i,j=1}^n(1-B_i)(1-B_j)(\func(x_i)-\func(x))(\func(x_j)-\func(x))}{(1-B_X)^2(n+2)^2}\right) \\
&+ \mathcal{O}\left(\frac{1}{(1-B_X) (n+2)} + \frac{\sigma^2}{(1-B_X)^2 (n+2)}\right) \\
&\leq L^2\Delta^2 + \mathcal{O}\left(\frac{1}{(1-B_X) (n+2)} + \frac{\sigma^2}{(1-B_X)^2 (n+2)}\right)
\end{align*}

Finally, by taking the expectation over experiment points $X$, we get:
\begin{align*}
\mathbb{E}_{X,\bf{s},\epsilon}&\left[\mathbb{E}((\tilde{\func}(x|\mathcal{S})-\func(x))^2)\right] = \sum_{n=0}^t Pr(D_x=n)\mathbb{E}_{S,\epsilon}\left[\mathbb{E}((\tilde{\func}^t(x)-\func(x))^2)  | D_x=n \right] \\
&\leq L^2\Delta^2  + \mathcal{O}\left(\frac{C^{(1)}}{\Delta^d t} + \frac{C^{(2)}\sigma^2}{\Delta^dt}\right) \\
\end{align*}
where $C^{(i)} = \mathbb{E}_X\left[\frac{1}{(1-B_X)^i}\right]$. Therefore, if we choose $\Delta = \frac{1}{L^{\frac{2}{d+2}}}t^{-\frac{1}{d+2}}$, then we obtain the desired result.

\end{proof}

	\section{Smooth Beta processes for classification} \label{sec:classification}
In this appendix, we extend the convergence rates in $L_2$ function approximation to $L_1$ and Bayes risk (misclassification error). These are to be understood as corollaries to the proofs presented in Sec.~\ref{sec-app:proofs}. Furthermore, we establish the connection between SBPs in the static setting and nearest neighbor techniques. However, our method allows for precise prior knowledge injection, whose efficiency is empirically demonstrated on a synthetic classification experiment.

\subsection{Convergence in $L_1$ norm}\label{sec-app:conv-l1}
Leaving out constants, Theorems~\ref{Static-convergence-theorem}, \ref{Dynamic-convergence-theorem}, and \ref{Dynamic-convergence-theorem-general} provide convergence rates of the type $\mathcal{O} \left(t^{-\frac{2}{d+2}} \right)$. In all three settings, we obtain the following corollary for the error in $L_1$ norm:

\begin{corollary}[Convergence in $L_1$]\label{L1-theorem}
Under the assumptions of Theorems~\ref{Static-convergence-theorem}, \ref{Dynamic-convergence-theorem}, and \ref{Dynamic-convergence-theorem-general}, the corresponding Algorithms~\ref{algo-static-setting}, \ref{algo-dynamic-setting}, and \ref{algo-dynamic-setting-general} converge in $L_1$ norm to $\func(x)$:
	\begin{equation*}
		\sup_{x \in [0,1]^d}\mathbb{E}_{\mathcal{S}} \left(\mathbb{E} \left| \tilde{\func}(x|\mathcal{S}) - \func(x) \right| \right) = \mathcal{O}\left(t^{-\frac{1}{d+2}}\right),
	\end{equation*}
where we leave out the constants of the respective theorems.
\end{corollary}
\begin{proof}
	For all three cases, the statement follows from the application of Jensen's inequality. We have
	\begin{equation}
	\mathbb{E}_{\mathcal{S}} \left( \mathbb{E} \left| \tilde{\func}(x|\mathcal{S}) - \func(x) \right| \right) 
	= \mathbb{E}_{\mathcal{S}} \left( \mathbb{E} \left( \sqrt{\left( \tilde{\func}(x|\mathcal{S}) - \func(x) \right)^2} \right) \right)
	\leq \sqrt{ \mathbb{E}_{\mathcal{S}} \left( \mathbb{E} \left( \left( \tilde{\func}(x|\mathcal{S}) - \func(x) \right)^2 \right) \right)}, \label{eq-app:jensen}
\end{equation}
which yields the presented convergence rates by taking the square root of the rates of the respective Theorems for $L_2$ convergence.

\end{proof}

\subsection{Convergence in Bayes risk}
In the classification setting, it is natural to use the posterior predictive of the Beta-Bernoulli model. Therefore, we have the classifier $\tilde{s}(x|\mathcal{S})$ based on the posterior parameters $\tilde{\alpha}(x), \tilde{\beta}(x)$:
\begin{align}
	\tilde{s}(x|\mathcal S) = 
	\begin{cases}
		1 \qquad & \text{if }\frac{\tilde{\alpha}(x)}{\tilde{\alpha}(x)+\tilde{\beta}(x)} \geq 0.5, \\
		0 \qquad & \text{otherwise.}
	\end{cases}
	\label{eq:decision_approx}
\end{align}

To estimate the performance of a classifier, the agreement with the Bayes optimal classifier is used. The Bayes risk of a classification problem is minimized by the omniscient Bayes classifier:

\begin{definition}[Bayes risk and optimal classifier]\label{def:bayesopt}
For any $x\in \mathcal{X}$, the Bayes risk of a classifier $\tilde{s}:\mathcal{X}\rightarrow \{0,1\}$ is given by
\begin{equation}
	R(\tilde{s}, x) = \mathbb{P}_{s\sim \text{B}(\func(x))} \left[s \neq \tilde{s}(x)\right].
\end{equation}
The Bayes optimal classifier is given based on the underlying probability function $\func(x)$. The corresponding decision rule is
\begin{equation}
    s^*(x) = \mathbbm{1}_{\func(x)\geq0.5},
\end{equation}
where $\mathbbm{1}_{\{\cdot\}}$ denotes the indicator function. This decision rule incurs the following optimal Bayes risk:
\begin{equation}
    R^*(x) = R(s^*, x) = \min \{ \func(x), 1-\func(x) \}.
\end{equation}
\end{definition}

To relate the convergence in $L_1$ to Bayes risk, the following simple Lemma is useful and allows to establish convergence in Bayes risk in Thm.~\ref{Bayes-conv}.

\begin{lemma}\label{lem:clfl1}
Suppose $\text{B}(\cdot)$ denotes a Bernoulli distribution, $p, q \in [0, 1]$ and $s' \in \{0, 1\}$. Then we have
\begin{equation}
    \mathbb{P}_{s \sim \text{B}(p)} \left[ s \neq s' \right] \leq \mathbb{P}_{s \sim \text{B}(q)} \left[ s \neq s' \right] + \left| p - q \right|, \label{eq:boundclfl1}
\end{equation}
which relates the misclassification directly to $\ell_1$ loss.
\end{lemma}
\begin{proof}
	Suppose $s'=1$. Then the left-hand side is $p$ and the right-hand side gives $q + |p-q|$. If $p>=q$, we have for the right-hand side $q + p - q = p$ and equality holds. If $p<q$, we have for the right-hand side $q + q - p$ and $2p \leq 2q$ by the assumption $p<q$. The same argument works for $s'=0$ by symmetry.
\end{proof}

\begin{theorem}[Convergence in Bayes risk]\label{Bayes-conv}
Under the assumptions of Theorems~\ref{Static-convergence-theorem}, \ref{Dynamic-convergence-theorem}, and \ref{Dynamic-convergence-theorem-general}, the classifier in Eq.~\eqref{eq:decision_approx} based on the posterior parameters obtained by the corresponding Algorithms~\ref{algo-static-setting}, \ref{algo-dynamic-setting}, and \ref{algo-dynamic-setting-general} uniformly converges to the risk of the Bayes optimal classifier $s^*$, i.e. for any $x\in \mathcal{X}$:
	\begin{equation}
 		\mathbb{E}_{\mathcal S} \left[ R(\tilde{s}, x) \right]  
		\leq R^*(x) + \mathcal{O}\left(t^{-\frac{1}{d+2}}\right), \label{eq:class_guarantee}
	\end{equation}\end{theorem}
	where constants of the respective theorems are left out (see Sec.~\ref{sec-app:conv-l1}).
\begin{proof}
Using Lemma~\ref{lem:clfl1}, we have the following for any $x \in \mathcal{X}$:
\begin{align}
    R(\tilde{s}, x) = 
    \mathbb{P}_{s \sim \text{B}(\func(x))} \left[ s \neq \tilde{s}(x|\mathcal{S}) \right]
    &\leq \mathbb{P}_{s \sim \text{B}(\tilde{\func}(x|\mathcal{S}))} \left[ s \neq \tilde{s}(x|\mathcal{S}) \right] + \left |\tilde{\func}(x|\mathcal{S}) - \func(x) \right| \nonumber \\
    &= \min \{ \tilde{\func}(x|\mathcal{S}), 1-\tilde{\func}(x|\mathcal{S})\} + \left |\tilde{\func}(x|\mathcal{S}) - \func(x) \right| \nonumber \\
    &\leq  \min \{ \func(x), 1-\func(x) \} + 2 \left |\tilde{\func}(x|\mathcal{S}) - \func(x) \right| \nonumber \\
    &= R^*(x) + 2 \left |\tilde{\func}(x|\mathcal{S}) - \func(x) \right|
\end{align}
Now, we can apply the convergence in $L_1$ of Corollary~\ref{L1-theorem} and get the desired result:
\begin{align}
    \mathbb{E}_{\mathcal{S}} \left( \mathbb{P}_{s\sim \text{B}(\func(x))} \left( s \neq \tilde{s} (x|\mathcal{S}) \right) \right) \leq R^*(x) + \mathcal{O}\left(t^{-\frac{1}{d+2}}\right).
\end{align}
\end{proof}

\subsection{Related methods and practical considerations}
Smooth Beta processes are designed for probability function approximation, in which case the estimation of the standard deviation on top of the function approximation is useful. In the particular static classification setting, SBPs are tightly connected to the fixed-radius nearest neighbors (\emph{NN}) classifier. SBPs have the advantage to specify a prior, which is useful to incorporate knowledge or combat biased data. In contrast to fixed-radius NN, SBPs perform additive smoothing like the famous Krichevsky-Trofimov estimator~\cite{krichevsky1981performance} by adding \emph{pseudo-counts}. Despite the introduced bias, SBPs converge optimally to the Bayes classifier: the rate proven in Thm.~\ref{Bayes-conv} matches the lower-bound established by \citet{audibert2007fast} for classification. 

On a practical side, faster inference methods are available due to the algorithmic fixed-radius nearest neighbors problem. Both exact (e.g. $k$-d and ball trees) and approximate (e.g. hashing-based) methods can be used for faster inference schemes. For further practical considerations and background on the fixed-radius NN algorithm, we refer to \citet{chen2018explaining}.

We conduct a synthetic experiment in order to show how the specification of a prior can help in the low data regime. We compare the convergence of SBP with various priors and the standard fixed-radius NN algorithm. For an informative prior, we set the prior $\tilde{\func}(x) \sim \text{Beta}(\alpha(x), \beta(x))$ such that $\mathbb{E}[\tilde{\func}(x)] = \func(x)$ and $\mathbb{V}[\tilde{\func}(x)] = v$. In Fig.~\ref{fig:prior-conv}, we compare the convergence for different values of $v$: in the low data regime, SBPs can profit strongly from an informative prior. With increasing number of observations, the approximation quality varies less as we expect it to happen for a Bayesian method. Asymptotically, the convergence rate is the same.

\begin{figure*}[t!]
	\centering
		\includegraphics[width=6cm]{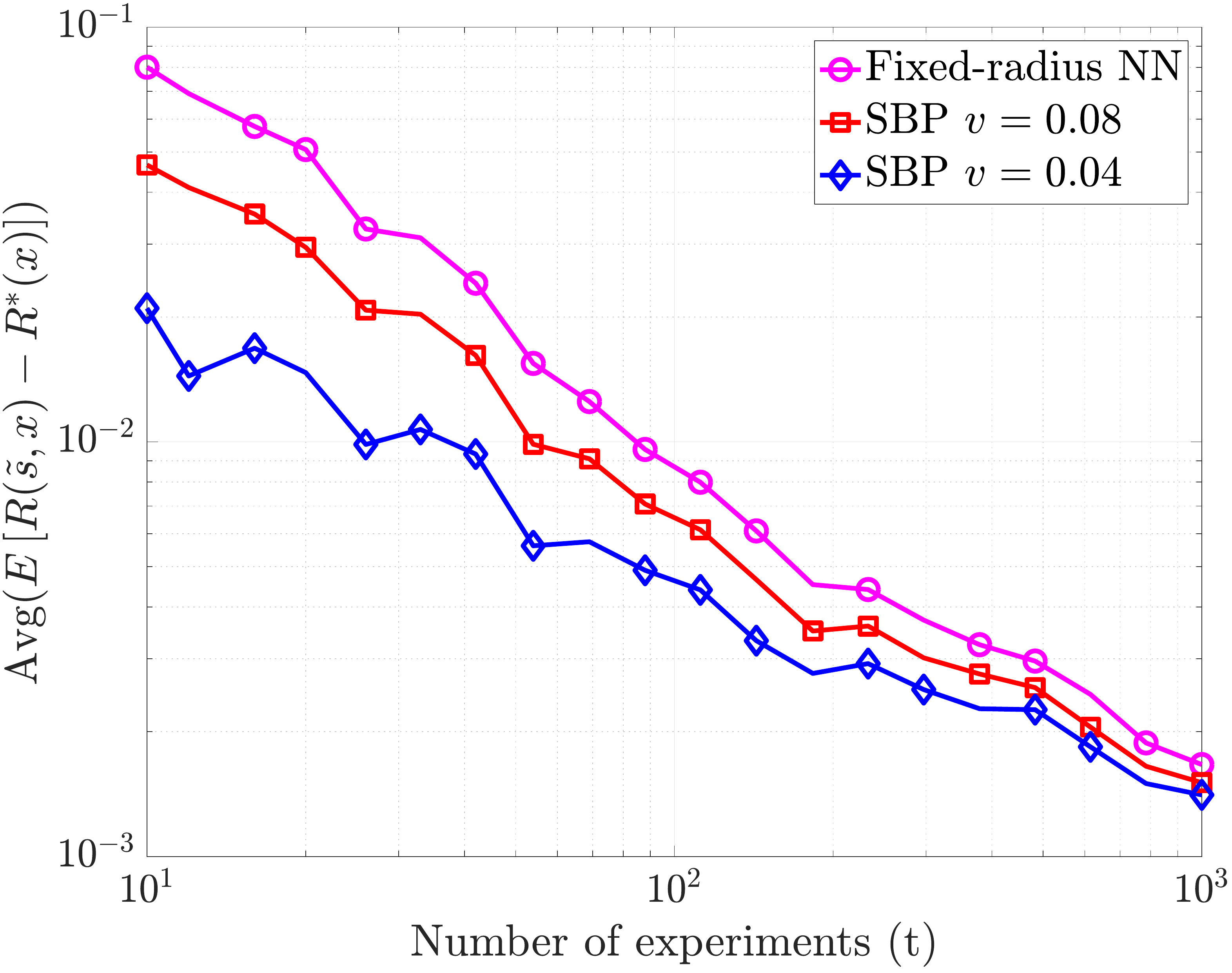}
		\caption{Bayes risk of SBP with specified informative prior, which is identical to the underlying function $\func(x)$, compared to fixed-radius NN which can not specify a prior in its standard framework.}
		\label{fig:prior-conv}
\end{figure*}

 }
 \fi

\end{document}